\documentclass[iicol]{sn-jnl}
\usepackage[authoryear,sort&compress]{natbib}%

\usepackage[switch]{lineno}

\newenvironment{lfigure}{%
  \begin{figure*}%
}{%
  \end{figure*}%
}
\newenvironment{sfigure}{%
  \begin{figure}%
}{%
  \end{figure}%
}

\algrenewcommand\algorithmicrequire{\textbf{Input:}}
\algrenewcommand\algorithmicensure{\textbf{Output:}}

\usepackage[labelformat=simple]{subcaption}

\usepackage{mathtools}

\usepackage{caption}
\usepackage{amsmath}

\newcommand{\mycite}[1]{\citep{#1}}
\newcommand{\Real}{\mathbb{R}}
\newcommand{\eps}{\varepsilon}
\newcommand{\tr}{^{\top}}

\newcommand{\pp}[1]{\frac{\partial}{\partial{#1}}}

\newcommand{\AS}[1]{\left(#1\right)}
\newcommand{\ASBR}[1]{\right.#1\left.}
\newcommand{\AD}[1]{\left[#1\right]}

\newcommand{\AN}[1]{\begin{align*}#1\end{align*}}
\newcommand{\AL}[1]{\begin{align}#1\end{align}}
\newcommand{\SP}[1]{\begin{split}#1\end{split}}
\newcommand{\LN}[1]{\label{#1}\\}
\newcommand{\LB}[1]{\label{#1}}
\newcommand{\NN}{\nonumber\\}
\newcommand{\argmin}{\mathop{\rm arg~min}\limits}
\newcommand{\argmax}{\mathop{\rm arg~max}\limits}
\newcommand{\mywidth}{3in}
\renewcommand{\vec}[1]{{\boldsymbol{#1}}}
\newcommand{\ph}{\varphi}

\newcommand{\fei}{m}
\newcommand{\fc}{\mset}
\newcommand{\fea}{\ph}
\newcommand{\summ}{\sum_{\fei\in\fc}}
\newcommand{\sumt}{\sum_{t=0}^{T}}

\newcommand{\et}{\eta}
\renewcommand{\th}{\theta}
\newcommand{\vth}{{\vec{\th}}}

\newcommand{\eng}{E}
\newcommand{\ent}{H}
\newcommand{\ten}{\tilde{H}}
\newcommand{\entrho}{H_{\rho}}
\newcommand{\entcom}{H_{c}}
\newcommand{\loss}{F}
\newcommand{\losa}{\mathcal{L}}

\newcommand{\tloa}{\tilde{\losa}}
\newcommand{\lmbd}{\Lambda}
\newcommand{\lm}{\lmbd_\fei}
\renewcommand{\pm}{\fea_\fei}
\newcommand{\fm}{\eta_\fei}

\newcommand{\mm}{\mu_\fei}
\newcommand{\vtho}{{\vth^*}}
\newcommand{\thom}{{\th^*_\fei}}
\newcommand{\thm}{\th_\fei}

\newcommand{\am}{a_\fei}
\newcommand{\va}{\vec{a}}
\newcommand{\vw}{\vec{w}}
\newcommand{\wmp}{{w'}_\fei}
\newcommand{\wm}{w_\fei}

\renewcommand{\tt}{^{(T)}}
\newcommand{\ts}{^{(t)}}
\newcommand{\tn}{^{(T+1)}}
\newcommand{\tp}{^{(T-1)}}
\newcommand{\tmax}{T_{\textrm{max}}}
\newcommand{\ups}{\eps}

\newcommand{\xset}{\mathcal{X}}
\newcommand{\qset}{\mathcal{Q}}
\newcommand{\expe}[2]{\mathbb{E}_{#1}\AD{#2}}

\newcommand{\rst}{\rho^{t,T}}

\newcommand{\vx}{{\vec{x}}}
\newcommand{\cmark}{\checkmark}
\newcommand{\xmark}{-}
\newcommand{\qcur}{q_{cur}}
\newcommand{\herding}{herding}
\newcommand{\domx}{D}

\newcommand{\domb}{\domx}
\newcommand{\domc}{\mathcal{D}}
\newcommand{\domblb}{\domb^-}
\newcommand{\dombub}{\domb^+}
\newcommand{\lbx}{\alpha}
\newcommand{\ubx}{\beta}

\newcommand{\lbb}{\lbx}

\newcommand{\lbd}{\lbx'}

\newcommand{\ubb}{\ubx}

\newcommand{\ubd}{\ubx'}
\newcommand{\etalba}{A}
\newcommand{\etauba}{B}
\newcommand{\etalbb}{\bar{A}}
\newcommand{\etaubb}{\bar{B}}
\newcommand{\setp}[1]{\{{#1}\}}
\newcommand{\setpar}[2]{\{{#1}\mid{#2}\}}
\newcommand{\bj}{\boldsymbol{J}}
\newcommand{\mset}{\mathcal{M}}
\newcommand{\numm}{M}
\newcommand{\mdim}{\Real^{\numm}}
\newcommand{\eq}[1]{Eq.~(\ref{#1})}

\newcommand{\poutput}{p_{\textrm{output}}}
\newcommand{\toutput}{T_{\textrm{output}}}
\newcommand{\epsherding}{\eps_{\textrm{herding}}}
\newcommand{\pjump}{p_{\textrm{jump}}}
\newcommand{\tburnin}{T_{\textrm{burnin}}}
\newcommand{\pbm}{p_{\textrm{BM}}}
\newcommand{\pbi}{p_{\textrm{bi}}}
\newcommand{\kupdate}{k_{\textrm{update}}}
\newcommand{\etalearn}{\eta_{\textrm{learn}}}

\jyear{2021}%

\theoremstyle{thmstyleone}%
\newtheorem{theorem}{Theorem}

\theoremstyle{thmstyletwo}%
\newtheorem{corollary}[theorem]{Corollary}%
\newtheorem{example}{Example}%

\theoremstyle{thmstylethree}%

\begin{document}

\title[Entropic Herding]{Entropic Herding}


\author*[1]{\fnm{Hiroshi} \sur{Yamashita}}\email{hiroshi-yamashita@g.ecc.u-tokyo.ac.jp}
\author[2]{\fnm{Hideyuki} \sur{Suzuki}}
\author[3]{\fnm{Kazuyuki} \sur{Aihara}}

\affil*[1]{\orgdiv{Intelligent Mobility Society Design, Social Cooperation Program, Graduate School of Information Science and Technology, the University of Tokyo}, \orgaddress{\street{7-3-1 Hongo}, \city{Bunkyo-ku, Tokyo}, \country{Japan}}}

\affil[2]{\orgdiv{Graduate School of Information Science and Technology}, \orgname{Osaka University}, \orgaddress{\street{1-5 Yamadaoka}, \city{Suita-shi, Osaka}, \country{Japan}}}

\affil[3]{\orgdiv{International Research Center for Neurointelligence}, \orgname{the University of Tokyo}, \orgaddress{\street{7-3-1 Hongo}, \city{Bunkyo-ku, Tokyo}, \country{Japan}}}


\abstract{
Herding is a deterministic algorithm used to generate data points that can be regarded as random samples satisfying input moment conditions. The algorithm is based on the complex behavior of a high-dimensional dynamical system and is inspired by the maximum entropy principle of statistical inference. In this paper, we propose an extension of the herding algorithm, called entropic herding, which generates a sequence of distributions instead of points. Entropic herding is derived as the optimization of the target function obtained from the maximum entropy principle. Using the proposed entropic herding algorithm as a framework, we discuss a closer connection between {\herding} and the maximum entropy principle. Specifically, we interpret the original herding algorithm as a tractable version of entropic herding, the ideal output distribution of which is mathematically represented. We further discuss how the complex behavior of the herding algorithm contributes to optimization. We argue that the proposed entropic herding algorithm extends the application of herding to probabilistic modeling. In contrast to original herding, entropic herding can generate a smooth distribution such that both efficient probability density calculation and sample generation become possible. To demonstrate the viability of these arguments in this study, numerical experiments were conducted, including a comparison with other conventional methods, on both synthetic and real data.

}




\maketitle

\section{ Introduction  }
\label{sec-intro}
In a scientific study, we often collect data on the study object, but its perfect information can not be obtained. Therefore, we often have to make a statistical inference based on the available data, assuming a certain background distribution. One way of drawing a statistical inference is to adopt the distribution with the largest uncertainty, consistent with the available information. This is referred to as the maximum entropy principle \mycite{Jaynes} because entropy is often used to measure the uncertainty. 

Specifically, for a distribution that has a probability mass (or density) function $p(x)$, the (differential) entropy is defined as 
\AL{
H(p)=\expe{p}{-\log p(x)},
}
 where $\mathbb{E}_p$ denotes the expectation over $p$. Assume that the collected data are a set of averages $\mm$ of feature values $\pm(x)\in\Real$ for $m\in \mset$, where $\mset$ is a set of indices. Then, the distribution $p$ that the maximum entropy principle adopts is known as the Gibbs distribution: 
\AL{
\pi_{\vth}(x)=\frac{1}{Z_{\vth}}\exp\AS{-\summ\thm\pm(x)},\LB{eq-MRF}
}
 where $\vth$ is the vector consisting of the parameters $\theta_m$ for $m\in \mset$, and $Z_\vth$ is the coefficient that makes the total probability mass one. The parameters can be chosen by maximizing the likelihood of the data. However, if the distribution is defined in a high-dimensional or continuous space, parameter learning often becomes difficult because approximations of the averages over the distribution are required.

The herding algorithm \mycite{Herding,Herding2010} is another method that can be used to avoid parameter learning. This algorithm, which is represented as a dynamical system in a high-dimensional space, deterministically generates an apparently random sequence of data, where the average of feature values converges to the input target value $\mm$. The generated sequence is treated as a set of pseudo-samples generated from the background distribution in the following analysis. Thus, we can bypass the difficult step of parameter learning of the Gibbs distribution model. The paper by \citet{Herding} describes its motivation in relation to the maximum entropy principle.

In this paper, we propose an extension of {\herding}, which is called entropic herding. We use the proposed entropic herding as a framework to analyze the close connection between {\herding} and the maximum entropy principle. Entropic herding outputs a sequence of distributions, instead of points, that greedily minimize a target function obtained from an explicit representation of the maximum entropy principle. The target function is composed of the error term of the average feature value and the entropy term. We mathematically analyze the minimizer, which can be regarded as the ideal output of entropic herding. We also argue that the original herding is obtained as a tractable version of entropic herding and discuss the inherent characteristics of herding in relation to the complex behavior of herding.

In addition, in section \ref{sec-adv}, we argue the advantages of the proposed entropic herding, in contrast to the original, which we call point herding. The advantages include the smoothness of the output distribution, the availability of model validation based on the likelihood, and the feasibility of random sampling from the output distribution. 

These arguments are further discussed through numerical examples in Section \ref{sec-expe} using synthetic and real data.

\section{ Background  }
\label{sec-background}
\subsection{    Herding  }
\label{sec-herding}
  Let $\xset$ be the set of possible data points. Let us consider a set of feature functions $\pm:\xset\to\Real$ indexed with $m\in \mset$, where $\mset$ is the set of indices. Let $M$ be the size of $\mset$, and let the bold symbols denote $M$-dimensional vectors with indices in $\mset$. Let $\mm$ be the input target value of the feature mean, which can be obtained from the empirical mean of the features over the collected samples. Herding generates the sequence of points $x\tt\in\xset$ for $T=1,2,\ldots,\tmax$ according to the equations below: 
\AL{
x\tt &= \argmax_{x\in\xset} \summ\wm\tp\pm(x),\LB{eq-opt-point}\\
\wm\tt&=\wm\tp+\mm-\pm(x\tt)\LB{eq-update-point},
}
 where $\wm\tt$ are auxiliary weight variables, starting from the initial values $\wm^{(0)}$. These equations can be regarded as a dynamical system for $\wm\tt$ in $M$-dimensional space. The average of the features over the generated sequence converges as 
\AL{
    \frac{1}{\tmax}\sum_{T=1}^{\tmax}\pm(x\tt)\to\mm\LB{eq-herding-average}
}
 at the rate of $O(1/\tmax)$ \mycite{Herding}. Therefore, the generated sequence reflecting the given information can be used as a set of pseudo-samples from the background distribution. This convergence is derived using the equation $\wm^{(\tmax)}-\wm^{(0)}=\tmax\mm-\sum_{T=1}^{\tmax}\pm(x\tt)$ and the boundedness of $\vw^{(\tmax)}$. The equation is obtained by summing both sides of \eq{eq-update-point} for $T=1,\ldots,\tmax$. The boundedness of $\vw^{(\tmax)}$ is guaranteed under some mild assumptions by using the optimality of $x\tt$ in the maximization of \eq{eq-opt-point} \mycite{Herding}.

The probabilistic model (\ref{eq-MRF}) is also regarded as a Markov random field (MRF). Herding has been extended to the case in which only a subset of variables in an MRF is observed \mycite{HerdingPOMRF}. In addition, it is combined with the kernel method to generate a pseudo-sample sequence for arbitrary distribution \mycite{KernelHerding}. The steps of {\herding} can be applied to the update steps of Gibbs sampling \mycite{HerdedGibbs}, and the convergence analysis is provided by \citet{HGweightsharing}.  

\subsection{    Related work  }
\label{sec-fwt}
Herding used the weighted average of features with time-varying weight values. This technique, which can be called time-varying feature weighting, is also employed in other research areas. It is mainly used to mitigate the problem of local minima often encountered in optimization.

The maximum likelihood learning of MRF (\ref{eq-MRF}) is often performed using the following gradient: 
\AL{
\pp{\thm} \log p(y) = -\pm(y)+\expe{p}{\pm(x) }.
}
 The expectation in the second term can be approximated using the Markov chain Monte Carlo (MCMC) \mycite{CD,PCD}. However, it is known that MCMC often suffers from slow mixing because of the sharp local minima in the potential function $E(x)=-\log p(x)+const.$. In the case of MRF (\ref{eq-MRF}), the potential function is $E(x)=\summ\thm\pm(x)$, which also has the form of a weighted sum of the features. The parameter $\thm$ changes over time during the learning process. \citet{FPCD} demonstrated that the efficiency of learning can be improved by adding extra time variation to the parameter. 

Combinatorial optimization is another example of the application of time-varying feature weighting. The Boolean satisfiability problem (SAT) is the problem of finding a binary assignment to a set of variables satisfying the given Boolean formula. A Boolean formula can be represented in a conjunctive normal form (CNF), which is a set of subformulae called clauses combined by logical conjunctions. This problem can be regarded as the minimization of the weighted sum of features, where features are defined by the truth values of clauses in the CNF. Several methods \mycite{breakout,DLM00,PAWS,SWCC} solve the SAT problem by repeatedly improving the assignment for the target function and changing its weights when the process is trapped in a local minimum. 

\citet{ERTNPhys,Sudoku} proposed a continuous-time dynamical system to solve the SAT problem, which implements the local improvement and time variation of the weight values. It is also shown that this system is effective in solving MAX-SAT, which is the optimization version of the SAT problem \mycite{CTDSMaxSAT}. As suggested by \citet{Herding}, one advantage of using deterministic dynamics is the possibility of efficient physical implementation. \citet{YinEfficient} designed an electric circuit implementing the SAT-solving continuous-time dynamical system and evaluated its performance.

\section{ Entropic herding and its target function representing maximum entropy principle  }
\label{sec-main}
We first present a summary of the proposed algorithm, which we refer to as entropic herding. Let us suppose the same situation as in Section \ref{sec-herding}. For simplicity, we assume that $\xset$ is continuous, although the arguments also hold when $\xset$ is discrete. For a distribution $q$ on $\xset$, let $\ent(q)=-\int_{\xset} q(x)\log q(x) dx$ be its differential entropy. Let $\pm$ be the feature functions and $\mm$ be the input target value of the feature mean. In addition, $\lm$ and $\eps\tt$ are also provided to the algorithm as additional parameters. The proposed algorithm is an iterative process similar to original herding, and it also maintains time-varying weights $\am$ for each feature. Each iteration of the proposed algorithm indexed with $T$ is composed of two steps as in the {\herding}: the first step is to solve the optimization problem, and the second is the parameter update based on the solution. The proposed algorithm outputs a sequence of distributions $(r\tt)$ instead of points as is the case with the original herding. The two steps in each time step, which will be derived later, are represented as follows: 
\AL{
r\tt &= \argmin_{q\in\qset} \AS{\summ \am\tp\fm(q)} - H(q), \LN{eq-opt-entropic}
\begin{split}
\am\tt&=\am\tp\\
&\quad+\ups\tt\AS{\lm(\fm(r\tt)-\mm)-\am\tp},
\end{split}\LB{eq-update-entropic}
}
 where $\qset$ is the set of candidate distributions of the output for each step and $\fm(q)\equiv \expe{q}{\pm(x)}$ is the feature mean for the distribution $q$. The pseudocode of the algorithm is provided in Algorithm \ref{pseudocode}. As will be described later, we can modify the algorithm by changing the set of candidate distributions $\qset$ and the optimization method, allowing the suboptimal solution for \eq{eq-opt-entropic}.

\begin{algorithm*}
\caption{Pseudocode of the entropic herding}
\label{pseudocode}
\begin{algorithmic}[1]
\Require
\Statex set of feature functions $\pm:\mathcal{X}\to\Real$,
\Statex target values of the feature means $\mm$,
\Statex weight parameters $\lm$,
\Statex a sequence of step size $\ups\tt$
\Statex
\State choose an initial distribution $r^{(0)}$ and let $\fm^{(0)}=\expe{r^{(0)}}{\pm(x)}$ and $\am^{(0)}=\lm(\fm^{(0)}-\mm)$
\For{$T=1\ldots\tmax$}
    \State perform several optimization steps for $\min_{q\in\qset} \AS{\summ \am\tp\expe{q}{\pm(x)}} - H(q)$ (Eq.~(\ref{eq-opt-entropic}))
    \State let $r\tt$ be the obtained distribution
    \State update $\am\tt=\am\tp+\ups\tt\AS{\lm(\expe{r\tt}{\pm(x)}-\mm)-\am\tp}$ (Eq.~(\ref{eq-update-entropic}))
\EndFor
\end{algorithmic}
\end{algorithm*}

\subsection{    Target function  }
\label{sec-losa}
The proposed procedure summarized above is derived from the minimization of a target function as presented below.

Let us consider the problem of minimizing the following function: 
\AL{
\losa(p)=\frac{1}{2}\AS{\summ\lm(\fm(p)-\mm)^2}-\ent(p).\LB{eq-loss}
}

This problem minimizes the difference between feature means over $p$ and the target vector and simultaneously maximizes the entropy of $p$. $\lm$ can be regarded as the weights for the moment conditions $\fm(p)=\mm$. When $\lm\to+\infty$, the problem becomes an entropy maximization problem that requires the solution to satisfy the moment condition exactly. Thus, this can be interpreted as a relaxed form of the maximum entropy principle.  The function $\losa$ is convex because the negated entropy function $-H(p)$ is convex. 

Let $\pi_\vth$ be the Gibbs distribution, the density function of which is defined as \eq{eq-MRF}. Suppose that there exists $\vtho$ that satisfies the following equations 
\AL{
\thom=\lm(\fm(\pi_{\vtho})-\mm)\LB{eq-tho}.
}
 The conditions for the existence of the solution are provided in Appendix \ref{sec-proof}. The functional derivative $\frac{\delta\losa}{\delta p}(x)$ of $\losa$ at $\pi_\vtho$ is obtained as follows: 
\AL{
\begin{split}
\left.\frac{\delta\losa}{\delta p}(x)\right\vert_{p=\pi_\vtho}&=\summ \lm(\fm(\pi_\vtho)-\mm) \pm(x)\\
&\quad+ \log \pi_\vtho(x)+1
\end{split}\\
&=\summ \thom \pm(x) + \log \pi_\vtho(x)+1\\
&=-\log Z_\vtho +1.
}
 Suppose that we perturb the distribution $\pi_\vtho$ to $p$. Then, because $\int_{\xset}(\pi_\vtho(x)-p(x))dx=1-1=0$, the inner product $\langle \frac{\delta\losa}{\delta p}, p-\pi_\vtho \rangle$ becomes zero, where the inner product is defined as $\langle f, g\rangle=\int_{\xset} f(x)g(x)dx$. Therefore, $\pi_\vtho$ is the optimal distribution for the problem because of the convexity of $\losa$. If $\lm$ is sufficiently large, the feature mean $\fm(\pi_\vtho)$ for the optimal distribution is close to the target value $\mm$ from \eq{eq-tho}. 

\subsection{    Greedy minimization  }
\label{sec-greedy}
Let us consider the following greedy construction of the solution for the minimization problem of $\losa$. Let $r\tt$ be a distribution that is selected at time $T$ and $p$ be the distribution to be constructed as the weighted average of the sequence of distributions $(r\tt)$: 
\AL{
p(x) = \frac{\sum_{T=0}^{\tmax} {\rho}\tt r\tt(x)}{\sum_{T=0}^{\tmax}{\rho}\tt},\LB{eq-search-space}
}
 where $\rho\tt$ are given as a fixed sequence of parameters. We greedily construct $p$ by choosing $r\tt$ for each time step $T=1,\ldots,\tmax$ by minimizing the tentative loss value $\losa\tt\equiv\losa(p\tt)$ for  
\AL{
p\tt(x)=\sumt\rst r\ts(x),
}
 where $\rst=\rho\ts/(\sumt \rho\ts)$. The feature mean $\fm(p\tt)$ for each step is also represented as a weighted sum: 
\AL{
\fm(p\tt)=\sumt\rst \fm(r\ts).
}
 For particular types of weight sequences $(\rho\tt)$, the feature means can be iteratively computed based on previous time instances as follows: 
\AL{
\begin{split}
\fm(p\tt)&=\fm(p\tp)\\
&\quad+\ups\tt(\fm(r\tt)-\fm(p\tp)).\LB{eq-update-fm}
\end{split}
}
 For example, if the coefficients geometrically decay at the rate $\rho\in(0,1)$ as $\rho\tt=\rho^{\tmax-T}$ for $T>0$ and $\rho^{(0)} = \rho^{\tmax} / (1-\rho)$, we can update the feature mean with $\ups\tt=1-\rho$. If $\rho\tt$ is constant over $T$, then we set $\ups\tt=\frac{1}{T+1}$.

We approximate the entropy term $H(p\tt)$ by the weighted sum of the components: 
\AL{
\ten\tt=\sum_{t=0}^{T} \rst H(r\ts).\LB{eq-ten}
}
 Then, similarly, the update equation for $\ten$ is obtained as 
\AL{
\ten\tt=\ten\tp+\ups\tt(\ent(r\tt)-\ten\tp).\LB{eq-update-ten}
}
 The minimization problem to be solved at time $T$ is given as follows: 
\AL{
\tloa\tt &=\AS{\summ\frac{1}{2}\lm(\fm(p\tt)-\mm)^2} - \ten\tt.\LB{eq-tloa}
}
 From the convexity of entropy, $\ten\tt$ is the lower bound of $H(p\tt)$; hence, $\tloa\tt$ is the upper bound of the original target function $\losa\tt$. 

The step of \eq{eq-opt-entropic} is derived from the greedy construction of the solution that minimizes $\tloa$ as follows. By using Eqs.~(\ref{eq-update-fm}) and (\ref{eq-update-ten}), the target function can be rewritten as 
\AL{ 
\SP{
    \tloa\tt &=
    \AS{\summ\frac{1}{2}\lm\AS{(\fm(p\tp)-\mm)
    \ASBR{\ASBR{\\&\qquad}}
    {}+\ups\tt(\fm(r\tt)-\fm(p\tp))}^2}
}
\NN
&\quad- \AS{\ten\tp+\ups\tt(\ent(r\tt)-\ten\tp)}.
}
 For all $T$, let us define  
\AL{
\am\tt \equiv \lm(\fm(p\tt)-\mm).\LB{eq-am}
}
 Then, $\tloa\tt$ also follows the update rule similar to \eq{eq-update-fm} up to the linear term with respect to $\ups\tt$: 
\AL{
\SP{
    \tloa\tt&=\AS{\summ\frac{1}{2\lm}\AS{\am\tp
    \ASBR{\ASBR{\\&\qquad}}
    {}+\ups\tt\lm(\fm(r\tt)-\fm(p\tp))}^2}\\
    &\quad- \AS{\ten\tp+\ups\tt(\ent(r\tt)-\ten\tp)}
}\LB{eq-tloa-decomp1} \\
\SP{
    &=\tloa\tp \\
    &\quad+ \ups\tt \summ\am\tp(\fm(r\tt)-\fm(p\tp))
    \\&\quad
    {}- \ups\tt(\ent(r\tt)-\ten\tp) +O((\ups\tt)^2)
}\\
\SP{
    &=\tloa\tp \\
    &\quad + \ups\tt \AS{\summ\am\tp\fm(r\tt)-\ent(r\tt)}\\
    &\quad - \ups\tt \AS{\summ\am\tp\fm(p\tp)-\ten\tp}\\
    &\quad +O((\ups\tt)^2).}
\LB{eq-update-tloa}
}
 If $\ups\tt$ is small, by neglecting the higher-order term $O((\ups\tt)^2)$, the optimization problem at time $T$ can be reduced to a minimization of $\loss\tp(r\tt)$, where 
\AL{
\loss\tp(q)=\AS{\summ\am\tp\fm(q)}-H(q).
}

From \eq{eq-update-fm} and (\ref{eq-am}), the coefficients $\am\tt$ follow the update rule as 
\AL{
&\quad\am\tt\NN
&=\am\tp+\ups\tt\lm(\fm(r\tt)-\fm(p\tp))\\
&=\am\tp+\ups\tt\AS{\lm(\fm(r\tt)-\mm)-\am\tp},
}
 which is equivalent to \eq{eq-update-entropic}.

In summary, the proposed algorithm is derived from the greedy construction of the solution of the optimization problem $\min \losa(q)$: it selects a distribution $r\tt$ in each step by minimizing the tentative loss. The minimization can be reduced to the minimization of $\loss\tp(r\tt)$ that has a parametrized form, and the algorithm updates the parameters by calculating their temporal differences. The algorithm thus derived, referred to as entropic herding, has a form similar to original herding. 

\section{ Entropy maximization in {\herding}  }
\label{sec-ent-max}
Let us further study the relationship between entropic herding, original herding, and the maximum entropy principle.  

\subsection{    Target function minimization and parameter learning of the Gibbs distribution  }
\label{sec-sch}
Figure \ref{schematic} is a schematic of the relationship between the $\losa$ minimization and parameter learning of the Gibbs distribution. Each probability distribution on $\xset$ is represented as a filled circle in the figure. Each moment condition can be represented as a vector $\vec{\mu}$, and the vector space consists of possible target vectors projected onto the horizontal direction of the schematic. The vertical axis represents the entropy of the distribution. For each condition, the set of distributions satisfying the moment condition $\pm(p)=\mm$ for all $m$ is denoted by a dotted line. The top-most point of the line is the distribution with the highest entropy value under the condition. In other words, for each moment condition $\vec{\mu}^i$, the Gibbs distribution $\pi_{\theta^{i}}$ is obtained by entropy maximization (red arrow) along the corresponding dotted line. The target distribution $\pi_{\hat\theta}$ adopted by the maximum entropy principle is the solution of entropy maximization under the moment condition of interest. The top-most points for various conditions form a family of Gibbs distributions represented as \eq{eq-MRF} and denoted by the black solid curve. The parameter learning finds this target along this curve (green arrow). In contrast, the entropic herding finds the target directly by minimizing the target function $\losa$ that represents both moment condition and entropy maximization (blue arrow). The search space, which is represented by \eq{eq-search-space}, is different from the family of Gibbs distrubitions. 

\begin{lfigure}[ htbp ]
\begin{center}
 
\includegraphics[ width=0.8\textwidth ]{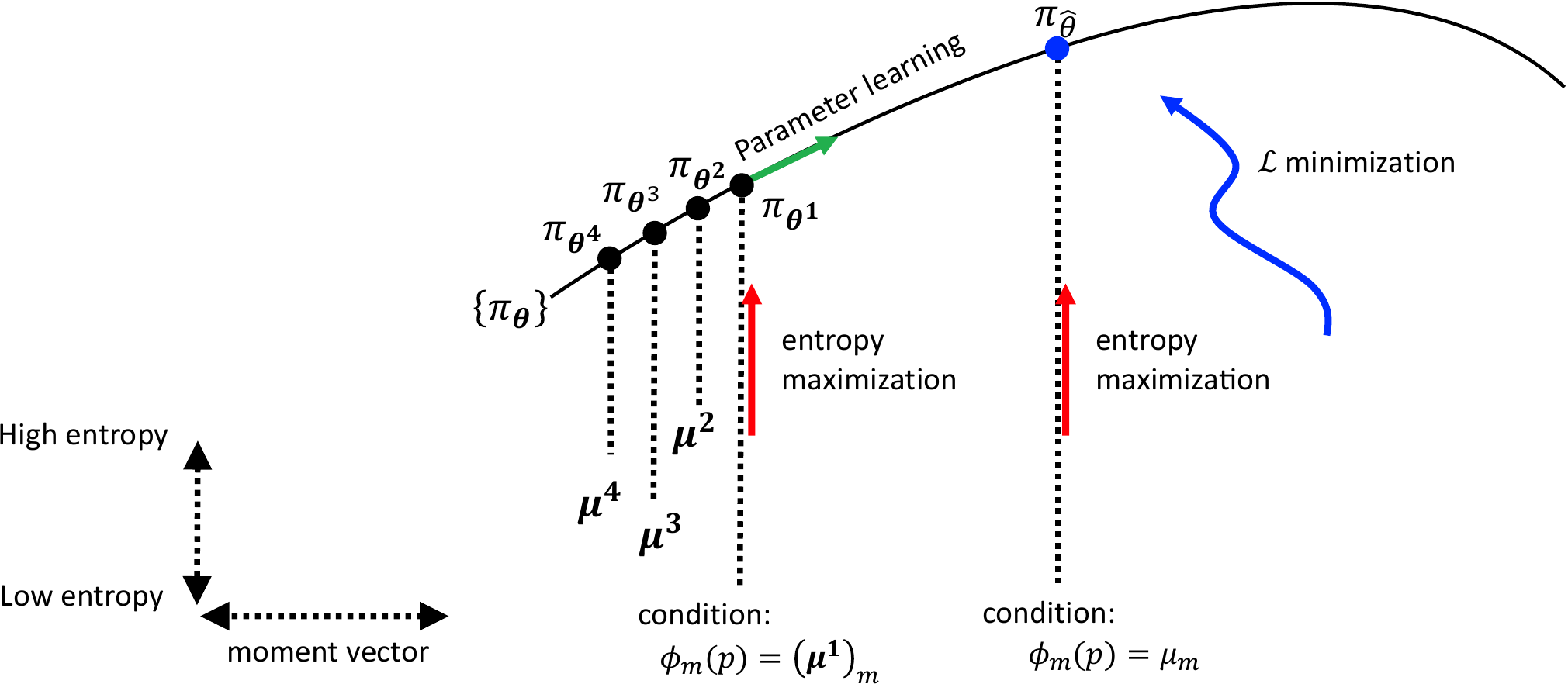} 
\end{center}
\caption{Schematic of the relationship between the $\losa$ minimization and the maximum entropy principle}
\label{schematic}
\end{lfigure} 

\subsection{    Tractable entropic herding  }
\label{sec-feasible}
From \eq{eq-update-tloa},  
\AL{
    \SP{
    \tloa\tt &\le \tloa\tp \\
    &\quad + \ups\tt\AS{\loss\tp(r\tt)-\loss\tp(p\tp)}
    }
}
 holds if we neglect $O((\ups\tt)^2)$ term, because $\ten\tp\le\ent(p\tp)$. If we can find the optimum distribution at time $T$, then $\loss\tp(r\tt)$ is less than $\loss\tp(p\tp)$ unless $p\tp$ itself is the optimum. Then, $\tloa\tt$ will be smaller than $\tloa\tp$. This corresponds to the theorem for {\herding} that $\sum ((1/\tmax)\sum_{T=1}^{\tmax}\pm(x\tt)-\mm)^2$ is bounded if we can find the optimum $x\tt$ in each optimization (\eq{eq-opt-point}) \mycite{Herding}.

Let $\pi_{\va\tp}$ be the Gibbs distribution defined as \eq{eq-MRF} with parameters $\vth = \va\tp$. The optimization problem of minimizing $\loss\tp(q)$ is equivalent to minimizing the Kullback–Leibler (KL) divergence $KL(q\|\pi_{\va\tp})$, the solution of which is given by $q=\pi_{\va\tp}$ as follows:  
\AL{
&\quad KL(q\|\pi_{\va\tp})\NN
&=\int q(x)\log\frac{q(x)}{\pi_{\va\tp}(x)} dx \\
\SP{
    &= \int q(x) \Biggl(\summ\am\tp\pm(x)\Biggr.\\
    &\qquad{}\Biggl. + \log Z_{\va\tp} + \log q(x)\Biggr) dx\\
}\\
&=\summ\am\tp\fm(q) - H(q) + \log Z_{\va\tp}\\
&=\loss\tp(q)+\log Z_{\va\tp}.\LB{eq-opt-kl}
}
 However, the Gibbs distribution is known to produce difficulties in applications such as the calculation of expectations. Fortunately, a suboptimal distribution $r\tt$ may be sufficient for decreasing $\tloa$, like the tractable version of the {\herding} introduced in \mycite{HerdingPOMRF}. Therefore, in practice, we can perform the optimization on a restricted set of distributions, as in variational inference, and allow a suboptimal solution for the optimization step (\eq{eq-opt-entropic}). Let $\qset$ denote the set of candidate distributions. We refer to this tractable version of the proposed algorithm as entropic herding, in contrast to the algorithm using the exact solution $r\tt=\pi_{\va\tp}$. 

\subsection{    Point herding  }
\label{sec-equiv}
Next, we show that the proposed algorithm can be interpreted as an extension of the original herding. Recall that, in the original herding in section \ref{sec-herding}, the optimization problem to generate the sample $x\tt$ and the update rule of weight $\wm\tt$ are represented in \eq{eq-opt-point} and (\ref{eq-update-point}), respectively.

Let us consider the proposed algorithm with the candidate distribution set $\qset$ restricted to point distributions. Then, the distribution $r\tt$ is the point distribution at $x\tt$, which is obtained by solving  
\AL{
x\tt = \argmin_{x\in\xset} \summ\am\tt\pm(x),
}
 because the entropy term of the optimization problem \eq{eq-update-entropic} can be dropped. Let us further suppose that $\ups\tt=\frac{1}{T+1}$ and $\lm$ are the same for all features. We introduce the variable transformation of weights as $\wmp\tn =-\frac{T+1}{\lm}\am\tt$. Then, the above optimization problem is equivalent to \eq{eq-opt-point} with the relationship $\wm=\wmp$. 

The update rule for $\wmp\tt$ is 
\AL{
\SP{
    \wmp\tn &=-\frac{T+1}{\lm}\AS{\am\tp+{}
    \ASBR{\\&\quad}
    \ups\tt\AS{\lm(\fm(r\tt)-\mm)-\am\tp}}
}\\
&=-\frac{T}{\lm}\am\tp-(\fm(r\tt)-\mm)\\
&=\wmp\tt+(\mm-\fm(r\tt)).
}
 Because $\fm(r\tt)=\pm(x\tt)$ holds, this update rule is also equivalent to \eq{eq-update-point}.

Therefore, the original herding can be interpreted as the tractable version of the proposed method by restricting $\qset$ to the point distributions. In the following, we refer to the original algorithm as point herding. In this case, the weights $\rho\tt$ in \eq{eq-search-space} are constant for all $T$ because we set $\ups\tt=\frac{1}{T+1}$. This agrees with the fact that the point herding puts equal weights on the output sequence while taking the average of the features, the difference of which with the input moment is minimized (see \eq{eq-herding-average}).

\subsection{    Diversification of output distributions by dynamic mixing in {\herding}  }
\label{sec-diversify}
The minimizer of \eq{eq-loss} is the Gibbs distribution $\pi_{\vtho}$, where $\vtho$ is given by \eq{eq-tho}, but there is a gap between the distribution obtained from the tractable entropic herding and the minimizer. This is because it solves the optimization problem with the approximate target function and by using sub-optimal greedy updates, as described in Sections \ref{sec-greedy} and \ref{sec-feasible}. Here, we discuss the inherent characteristics of entropic herding for decreasing the gap, which cannot be fully described as greedy optimization.

The target function for greedy optimization in (tractable) entropic herding is $\tloa$ (\eq{eq-tloa}). In deriving $\tloa$ from $\losa$, we introduce an approximation for the entropy term $\ent(p\tt)$ by its lower bound $\ten\tt$, which is the weighted average of the entropy values of the components as defined in Eq.~(\ref{eq-ten}). The gap between $\ent$ and $\ten$ can be calculated using the following equation: 
\AL{
&\quad H(p\tt)\NN
&=-\int p\tt(x)\log p\tt(x)dx\\
&=-\int\sumt\rst r\ts(x)\log p\tt(x)dx\\
\SP{
    &=-\int\sumt\rst r\ts(x)\Biggl(\log(\rst r\ts(x))\Biggr.\\
    &\qquad{} \Biggl. - \log\frac{\rst r\ts(x)}{p\tt(x)}\Biggr)dx
}\\
\SP{
    &=\sumt\rst\ent(r\ts)-\sumt \rst\log\rst\\
    &\quad+\int p\tt(x)\sumt\frac{\rst r\ts(x)}{p\tt(x)}\log\frac{\rst r\ts(x)}{p\tt(x)}dx,
}\\
&=\ten\tt+\entrho\tt-\expe{p\tt}{\entcom\tt(x)},
}
 where $\entrho\tt=-\sumt\rst\log\rst$, $\entcom\tt(x)=-\sumt\frac{\rst r\ts(x)}{p\tt(x)}\log\frac{\rst r\ts(x)}{p\tt(x)}$. Then, the gap is represented by 
\AL{
\ent(p\tt)-\ten\tt=\entrho\tt-\expe{p\tt}{\entcom\tt(x)}.
}
 If the coefficients $\rst$ are fixed, the gap is determined by $\entcom\tt$. If this gap is positive, we have an additional decrease in the true target function $\losa$ compared to the approximate target $\tloa$ of the greedy minimization.

Suppose that $t'$ is a random variable drawn with a probability of $\rho^{t',T}$, and $x$ is drawn from $r^{(t')}$. Then, $\entcom\tt(x)$ can be interpreted as the average entropy of the conditional distribution of $t'$, given $x$. If the probability weights assigned to $x$ by the components $t$, represented as $\rst r\ts(x)$, are unbalanced over $t$, then $\entcom\tt(x)$ is small so that the gap $\ent(p\tt)-\ten\tt$ is large. In contrast, if $r\ts$ are the same for all $t$ so that $r\ts=p\tt$, then $\entcom\tt(x)=\entrho\tt$ for all $x$. Then, the gap $\ent(p\tt)-\ten\tt$ becomes zero, which is the minimum value. That is, we achieve an additional decrease in $\losa$ if the generated sequence of distributions $r\ts$ is diverse. This can happen if $\va\ts$ are diverse, as the optimization problems in Eq.~(\ref{eq-opt-entropic}), defined by the parameters $\va\ts$, become diverse. In the original herding, the weight dynamics is weakly chaotic, and thus, it can generate diverse samples. We also expect that the coupled system of Eq.~(\ref{eq-opt-entropic}) and (\ref{eq-update-entropic}) in the high-dimensional space will exhibit complex behavior that achieves diversity in $\{\va\ts\}$ and $\{r\ts\}$. The extra decrease is bounded from above by $\entrho\tt$, which usually increases as $T$ increases. For example, if $\rst$ is constant over $t$, then $\entrho\tt=\log(T+1)$.

In summary, the proposed algorithm, which is a generalization of {\herding}, minimizes the loss function $\losa$ using the following two means: \begin{itemize} \item[(a) ] Explicit optimization: greedy minimization of $\tloa$, which is the upper bound of $\losa$, by solving Eq.~(\ref{eq-opt-entropic}) with the entropy term. \item[(b) ] Implicit diversification: extra decrease of $\losa$ from the diversity of the components $r\tt$ achieved by the complex behavior of the coupled dynamics of the optimization and the update step (Eq.~(\ref{eq-opt-entropic}) and (\ref{eq-update-entropic}), respectively). \end{itemize} The complex behavior of the herding contributes to the optimization implicitly through an increase of the gap $\ent-\ten$ by diversification of the samples. In addition, the proposed entropic herding can improve the output through explicit entropy maximization.

Thus, we can generalize the concept of herding by regarding it as an iterative algorithm having two components as described above. 

\subsection{    Probabilistic modeling with entropic herding   }
\label{sec-adv}
The extension of herding to the proposed entropic herding expands its application as follows.

The first important difference introduced by the extension is that the output becomes a mixture of probabilistic distributions. Thus, we can consider the use of entropic herding in probabilistic modeling. Let $\pi_{\hat\vth}$ be the distribution obtained from the maximum entropy principle. For each time step $T$, the tentative distribution $p\tt=\sumt\rst r\ts(x)$ is obtained by minimizing $\losa$, which represents the maximum entropy principle. Then, we can expect it to be close to $\pi_{\hat\vth}$ and use it as the probabilistic model $\poutput$ derived from the data. It should be noted that the model thus obtained $\poutput=p\tt$ is different from $\pi_{\hat\vth}$. It includes the difference between $\pi_{\hat\vth}$ and the minimizer $\pi_{\vtho}$ of $\losa$ by the finiteness of the weights $\lm$ in Eq.~(\ref{eq-loss}) and includes the difference between $\pi_\vtho$ and $\poutput$ because of the inexact optimization. We can also obtain another model $\poutput$ by further aggregating the tentative distributions $p^{(T)}$, expecting additional diversification. For example, we can use the average $\poutput=\frac{1}{\tmax}\sum_{T=1}^{\tmax}p\tt$ as the output probabilistic model. This is again a mixture of the output sequence $(r\tt)$ that only differs from $p^{(\tmax)}$ in the coefficients. A more specific method for output aggregation is described in Appendix \ref{sec-implementation}.

If $\qset$ is set appropriately and the number of components is not too large, the probability density function of the output $\poutput$ can be easily calculated. The availability of the density function value enables us to use likelihood-based model evaluation tools in statistics and machine learning, such as cross-validation. This can also be used to select the parameters of the algorithm such as $\lm$. 

If $\xset$ is continuous, we cannot use point herding to model the probability density because the sample points can only represent the non-smooth delta functions. Even if $\xset$ is discrete, when the number of samples $T$ is much smaller than the number of possible states $\lvert \xset \rvert$, the zero-probability weight should be assigned to a large fraction of states. The zero-probability weight causes infiniteness when the log-likelihood is calculated. In addition, the probability mass can only take a multiple of $1/T$, which causes inaccuracy especially for a state with a small probability mass. Many samples are required to obtain accurate probability mass values for likelihood-based methods. In contrast, the output of entropic herding has sample efficiency because each output component can assign non-zero probability mass to many states. This is demonstrated by numerical examples in the next section. 

Another important difference is that entropic herding explicitly optimizes the entropy term. Therefore, entropic herding can generate a distribution that has a higher entropy value than the point herding output. In addition, we can control how it focuses on the entropy or moment information of the data by changing the parameter $\lm$ in the target function. While point herding diversifies the sample points by its complex behavior, the optimization problem solved in each time step does not explicitly contain $\lm$ so that the balance between entropy maximization and moment error minimization is not controlled.

Moreover, if we restrict the candidate distribution set $\qset$ to be appropriately simple, we can easily obtain any number of random samples from $\poutput$ by sampling from $r\ts$ for each randomly sampled index $t$. This sample generation process can also be parallelized because each sample generation is independent. For example, we can use independent random variables following normal distributions or Bernoulli distributions as $\qset$, as shown in the numerical examples below.

Theoretically, the exact solution of the problem \eq{eq-opt-entropic} is obtained as $\pi_{\va\tp}$. However, we can enjoy the advantages of the entropic herding described above when $\qset$ is restricted to simple, analytic, and smooth distributions, even though the optimization performance is suboptimal.

In summary, entropic herding has both the merits of herding and probabilistic modeling using the distribution mixture.

\section{ Numerical Examples  }
\label{sec-expe}
In this section, we present numerical examples of entropic herding and its characteristics. We also present comparisons between entropic herding and several conventional methods. A detailed description of the experimental procedure is provided in Appendix \ref{sec-implementation}. The definitions of the detailed parameters, not described here, are also provided in the Appendix. 

\subsection{  One-dimensional bimodal distribution }

As a target distribution, we take a one-dimensional bimodal distribution 
\AL{
\pbi(x) &= \exp(-(x^4 - 3 x^2 + 0.5 x)) / Z,
}
 where $Z$ is the normalizing factor. The histogram of this distribution is shown by the orange bars in each panel of Fig.~\ref{hnm_hist}. For this distribution, we consider a set of four polynomial feature functions for {\herding} as follows: 
\AL{
\phi_i(x) = x^i\quad^\forall i\in\{1,2,3,4\}.
}
 From the moment values $\mm=\expe{\pbi}{\pm(x)}$, the maximum entropy principle reproduces the target distribution $\pbi$. Using this feature set, we ran point herding and entropic herding. Figure \ref{hnm_traj} shows the time series of the generated sequence, and Figure \ref{hnm_hist} shows the histogram compared to the target distribution $\pbi$. For entropic herding, we set the candidate distribution set $\qset$ as 
\AL{
\qset=\{\mathcal{N}(\mu,\sigma^2)\mid \mu\in\Real, \sigma > 0.01\},
}
 where $\mathcal{N}(\mu,\sigma^2)$ denotes a one-dimensional normal distribution with mean $\mu$ and variance $\sigma^2$.

Figures~\ref{hnm_traj_point} and \ref{hnm_hist_point} are for point herding. We observe the periodic sequence, and this causes a large difference in the histogram around $x=0$. Point herding has complex dynamics of the high-dimensional weight vector that achieves random-like behavior even in the deterministic system. However, in this case, the dimension of the weight vector is only four. This low dimensionality of the system can be a cause of the periodic output.

In contrast, Figs.~\ref{hnm_traj_entropy} and \ref{hnm_hist_entropy} show the results for entropic herding. Although this trajectory is periodic as well, the output in each step is a distribution with a positive variance such that it can represent a more diverse distribution. The difference in the histograms is reduced, as shown in Fig.~\ref{hnm_hist_entropy}.

In addition, the periodic behavior of point herding can be mitigated by introducing a stochastic factor into the system. Figures~\ref{hnm_traj_metro} and \ref{hnm_hist_metro} show the results of point herding with the stochastic Metropolis update (see Appendix \ref{sec-implementation} for details). We observed some improvements in the difference in the histogram. The stochastic update has the role of increasing the diversity of samples, as described in (b) in Section \ref{sec-diversify}. Thus, this algorithm is conceptually in line with entropic herding, although not included in the proposed mathematical formulation. Entropic herding, which generates a sequence of distributions, can also be combined with stochastic update. The results of this combination are shown in Figs.~\ref{hnm_traj_entropy_jump} and \ref{hnm_hist_entropy_jump}, respectively. We can see that the periodic behavior of point herding is diversified. The improvements in the difference between the histograms are significant especially for point herding (Fig.~\ref{hnm_hist_point} and \ref{hnm_hist_metro}).  

\begin{lfigure}[ htbp ]
\begin{center}
 
\begin{minipage}[b ]{\mywidth} {\includegraphics[ width=\textwidth ]{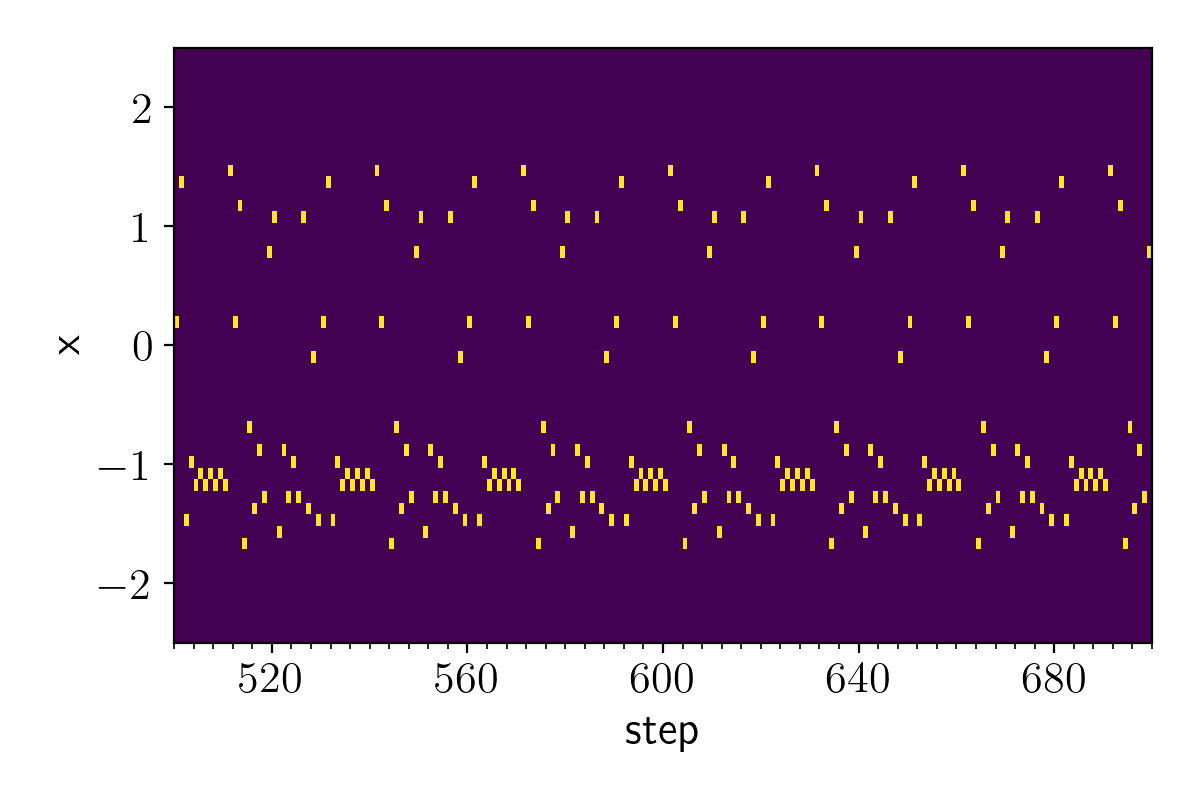} \subcaption{} \label{hnm_traj_point} }\end{minipage} 
\begin{minipage}[b ]{\mywidth} {\includegraphics[ width=\textwidth ]{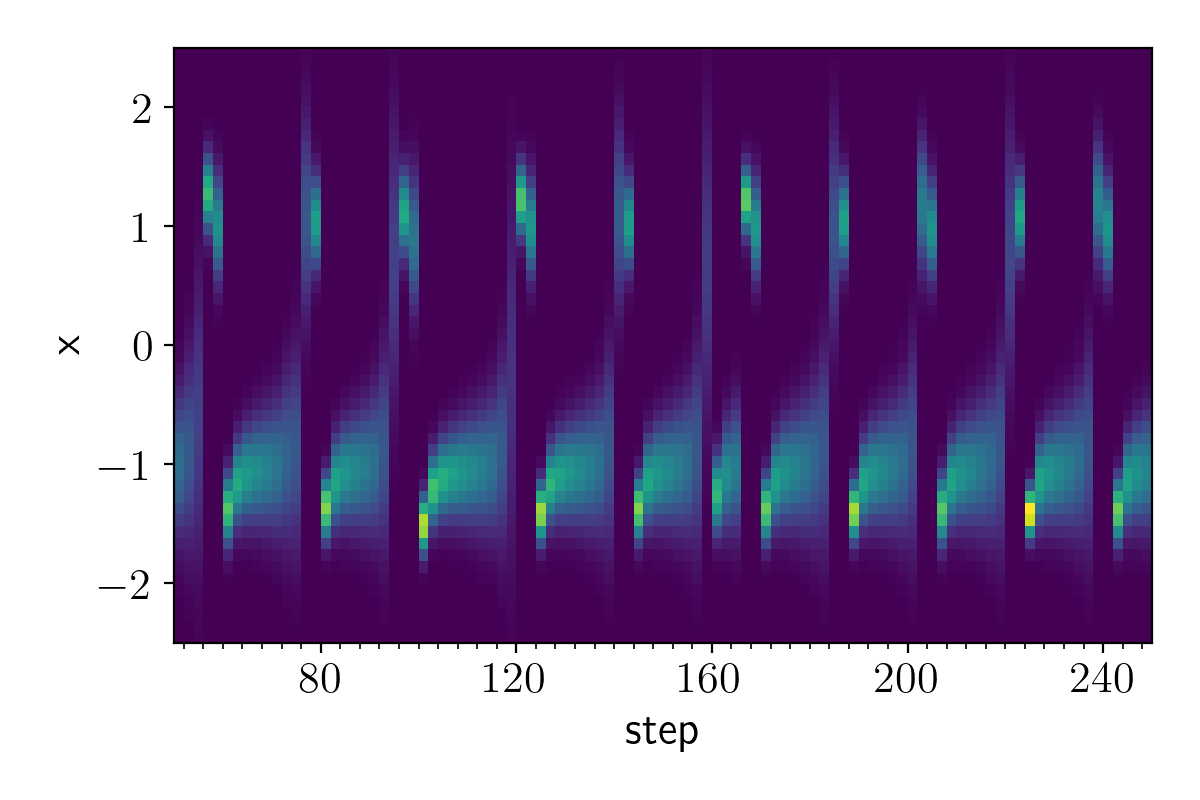} \subcaption{} \label{hnm_traj_entropy} }\end{minipage} \\ 
\begin{minipage}[b ]{\mywidth} {\includegraphics[ width=\textwidth ]{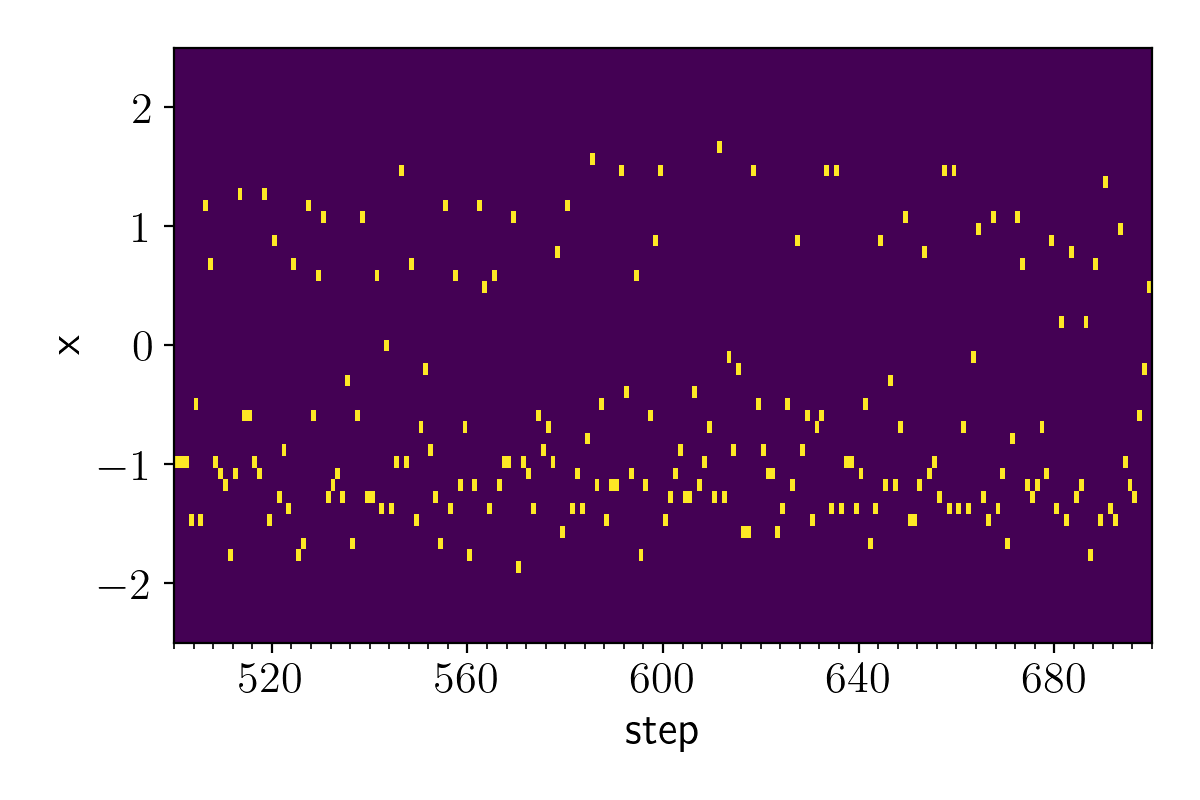} \subcaption{} \label{hnm_traj_metro} }\end{minipage} 
\begin{minipage}[b ]{\mywidth} {\includegraphics[ width=\textwidth ]{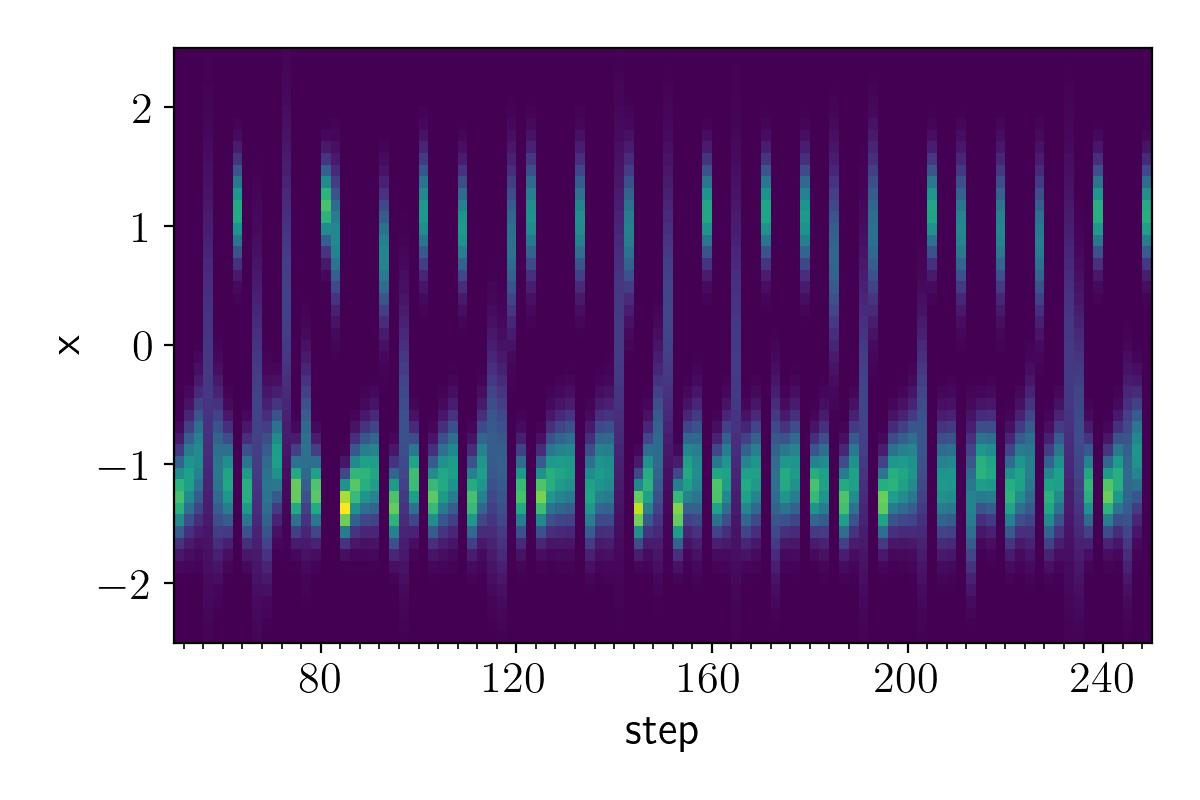} \subcaption{} \label{hnm_traj_entropy_jump} }\end{minipage} 
\end{center}
\caption{Output sequences of \subref{hnm_traj_point} point herding and \subref{hnm_traj_entropy} entropic herding for the one-dimensional bimodal distribution $\pbi$. The horizontal axis represents the index of the step denoted by $t$ in the text. The vertical axis represents the space $\xset=\Real$ and is divided into bins of width $0.1$. The color represents the probability mass for each bin, where yellow corresponds to a larger mass. The probability mass values were normalized to $[0, 1]$ in each plot before being colorized. Panels \subref{hnm_traj_metro} and \subref{hnm_traj_entropy_jump} show output sequences of point herding and entropic herding with stochastic optimization steps, respectively}
\label{hnm_traj}
\end{lfigure} 

\begin{lfigure}[ htbp ]
\begin{center}
 
\begin{minipage}[b ]{\mywidth} {\includegraphics[ width=\textwidth ]{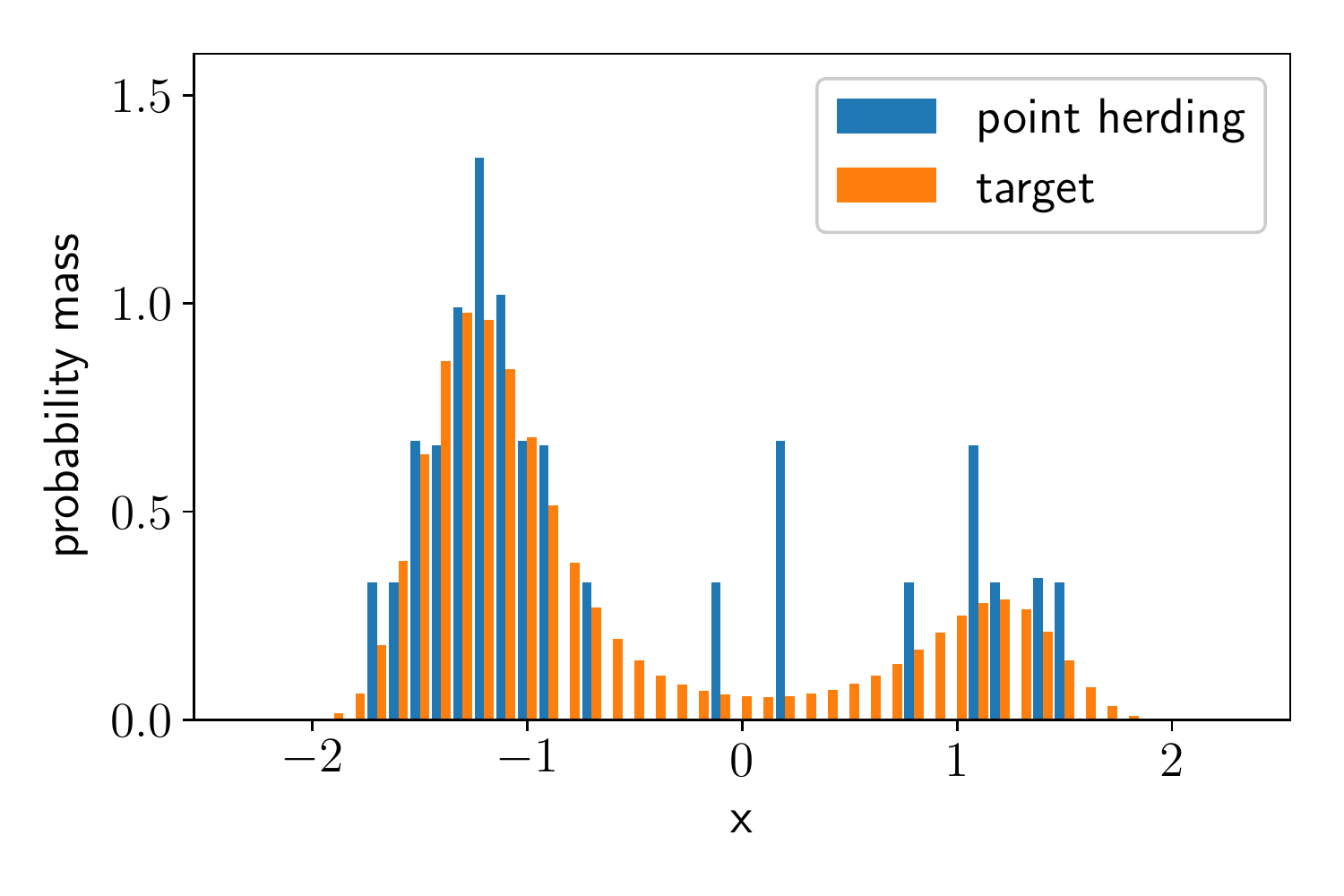} \subcaption{} \label{hnm_hist_point} }\end{minipage} 
\begin{minipage}[b ]{\mywidth} {\includegraphics[ width=\textwidth ]{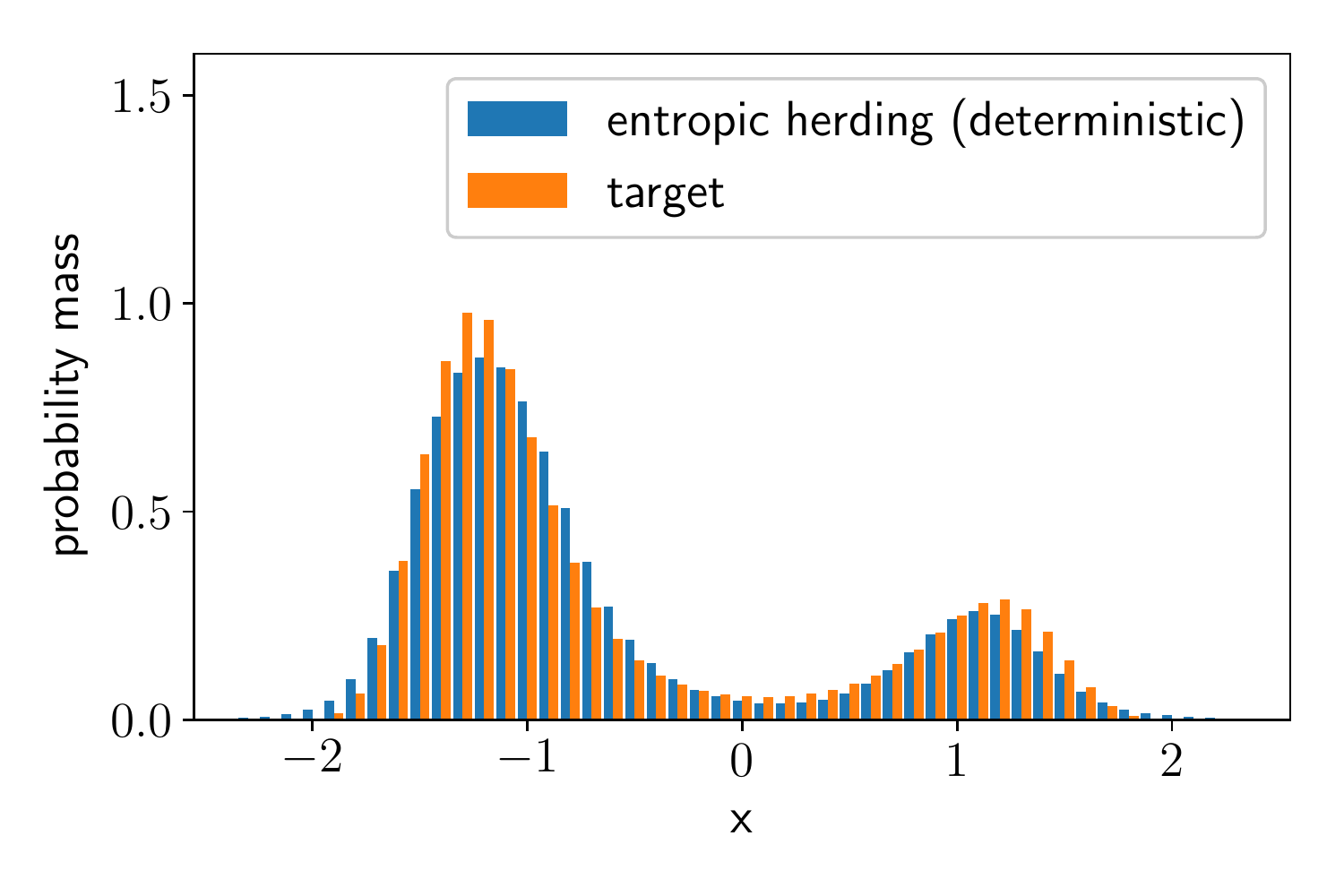} \subcaption{} \label{hnm_hist_entropy} }\end{minipage} \\ 
\begin{minipage}[b ]{\mywidth} {\includegraphics[ width=\textwidth ]{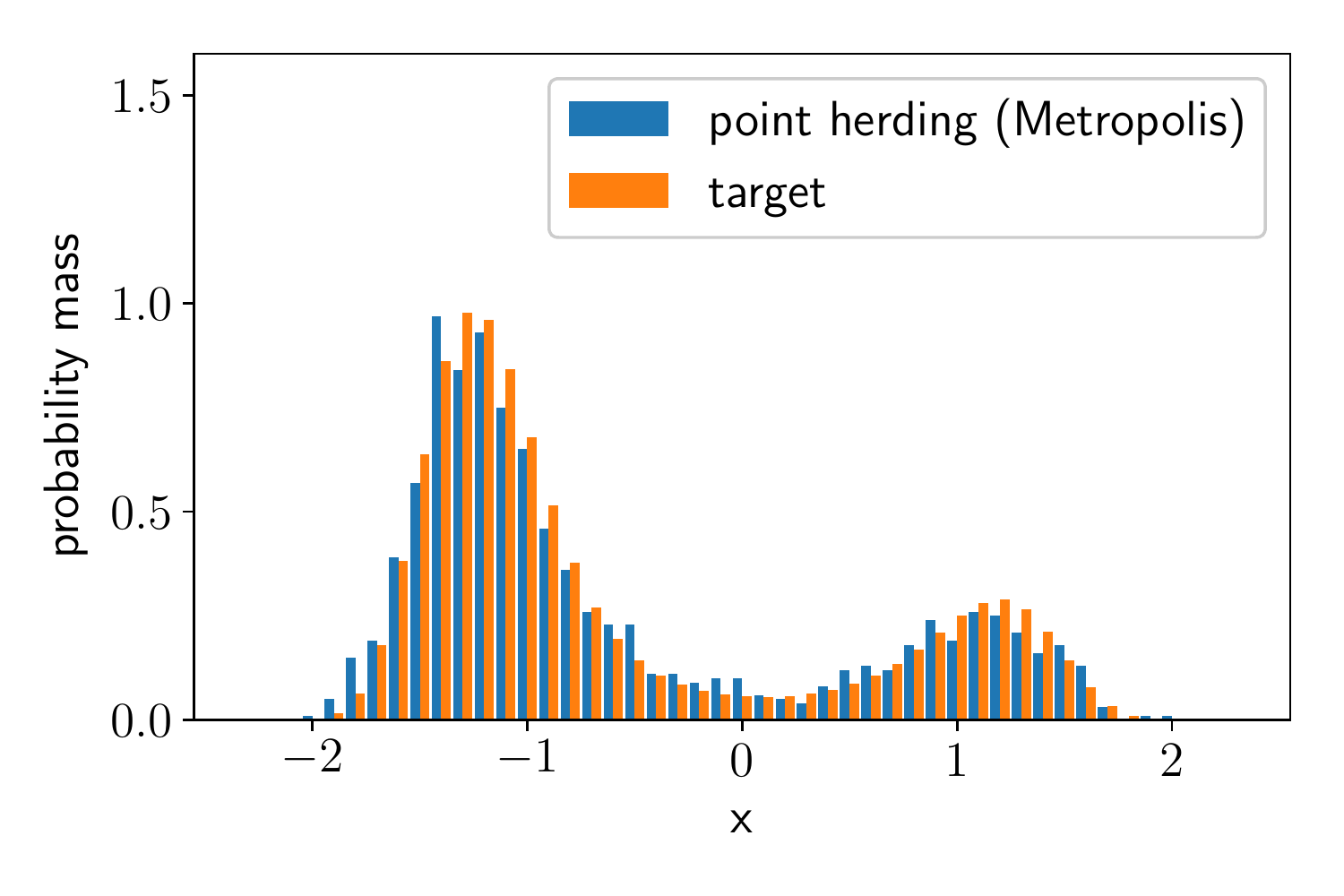} \subcaption{} \label{hnm_hist_metro} }\end{minipage} 
\begin{minipage}[b ]{\mywidth} {\includegraphics[ width=\textwidth ]{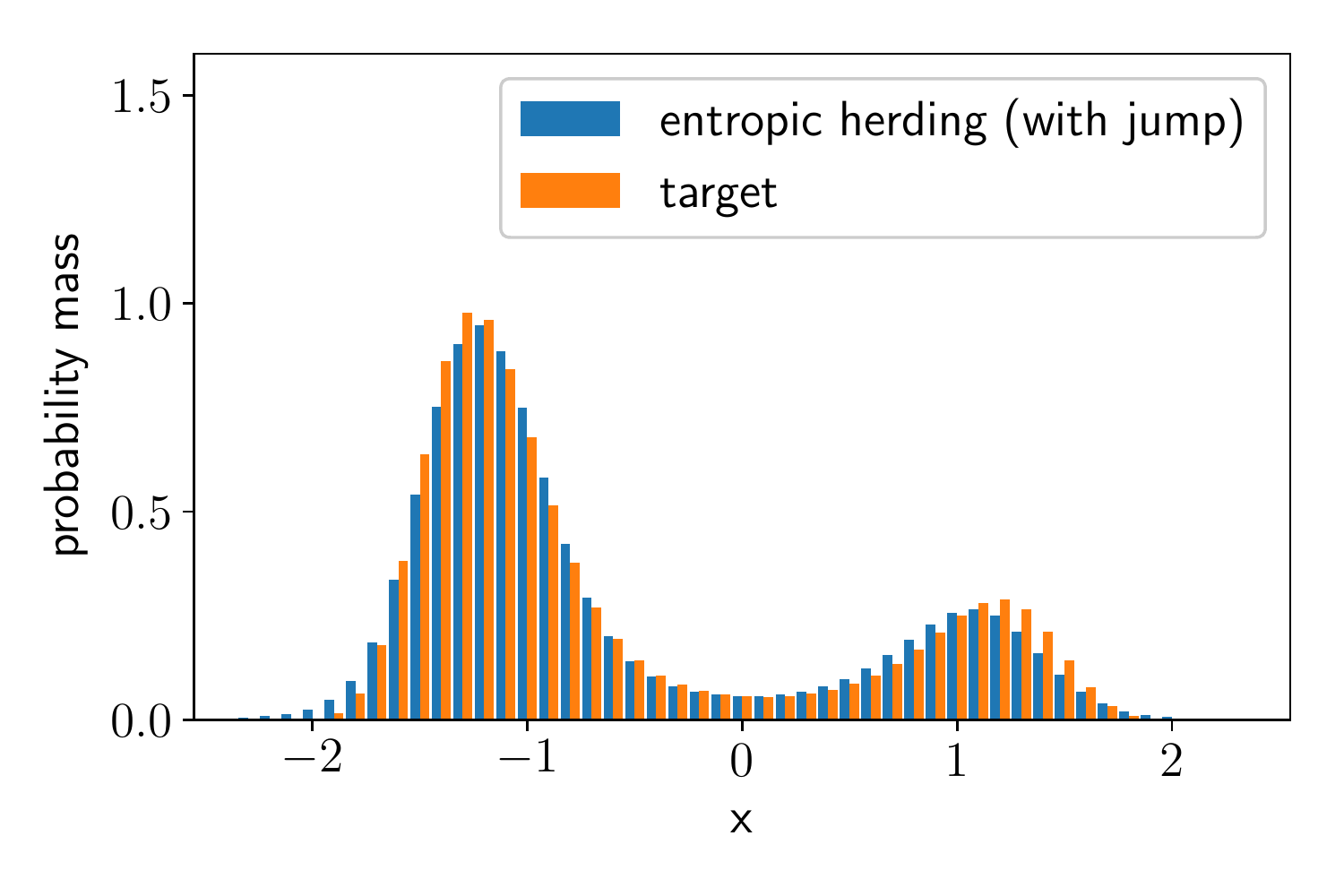} \subcaption{} \label{hnm_hist_entropy_jump} }\end{minipage} 
\end{center}
\caption{Histogram of the output distribution of \subref{hnm_hist_point} point herding and \subref{hnm_hist_entropy} entropic herding compared to the one-dimensional bimodal target distribution $\pbi$. The horizontal axis represents the space $\xset=\Real$ and is divided into bins of width $0.1$. The vertical axis represents the probability mass of each bin. Panels \subref{hnm_hist_metro} and \subref{hnm_hist_entropy_jump} show the output distributions of point herding and entropic herding with stochastic optimization steps, respectively}
\label{hnm_hist}
\end{lfigure} 

\subsection{  Boltzmann machine }

Next, we present a numerical example of entropic herding for a higher dimensional distribution. We consider the Boltzmann machine with $N=10$ variables. The state vector is $\vx=(x_1,\ldots,x_N)\tr$, where $x_i\in\{-1,+1\}$ for all $i\in\{1,\ldots, N\}$. The target distribution is defined as follows: 
\AL{
\pbm(\vx) = \frac{1}{Z}\exp\AS{-\sum_{i<j}W_{ij}x_ix_j},\label{eq-BM}
}
 where $W_{ij}$ are the randomly drawn coupling weights, and $Z$ is the normalizing factor. For simplicity, we did not include bias terms. For this distribution, we use a set of feature functions $\{\phi_{ij} \mid i<j\}$, as follows: 
\AL{
\phi_{ij}(x) = x_ix_j.
}
 The dimension of the weight vector is $N(N-1)/2$. With this feature set, we ran entropic herding and obtained the output sequence of 320 distributions. The input target value $\vec{\mu}$ was obtained from the feature mean $\expe{\pbm}{\phi_{ij}}$ calculated with the definition \eq{eq-BM}. We used the candidate distribution set $\qset$, defined as the set of the following distributions: 
\AL{
\SP{
&x_i = \left\{\begin{array}{l}
+1 \quad (\text{with probability } p_i)\\
-1 \quad (\text{with probability } 1-p_i)\\
\end{array}\right.\\
&\quad^\forall i\in\{1,\ldots, N\},
}
}
 where  $p_i\in [0, 1]$ are the parameters and $x_1,\ldots, x_N$ are independent. We also generated 320 identically and independently distributed random samples from the target $\pbm$ for comparison. 

Figure \ref{hbm_pdf} is a scatter plot comparing the probability mass between the output and the target distribution for each state. Figure~\ref{hbm_pdf_choice} shows the result for the empirical distribution of random samples. The Boltzmann machine has 1024 states for $N=10$, but we generated only 320 samples in this experiment. Therefore, the mass values for the empirical distribution have discrete values. In addition, there are many states with mass values of zero. We can observe the large difference in mass values between the empirical distribution and the target distribution.

In contrast, Fig.~\ref{hbm_pdf_eher} shows the results of entropic herding. As in the above example for bimodal distribution, each output of entropic herding can represent a more diverse distribution than a single sample. For most of the states that have large probability weights, the difference in mass value is within a factor of 1.5. This is much smaller than that in the case of the random samples. 

\begin{lfigure}[ htbp ]
\begin{center}
 
\begin{minipage}[b ]{\mywidth} {\includegraphics[ width=\textwidth ]{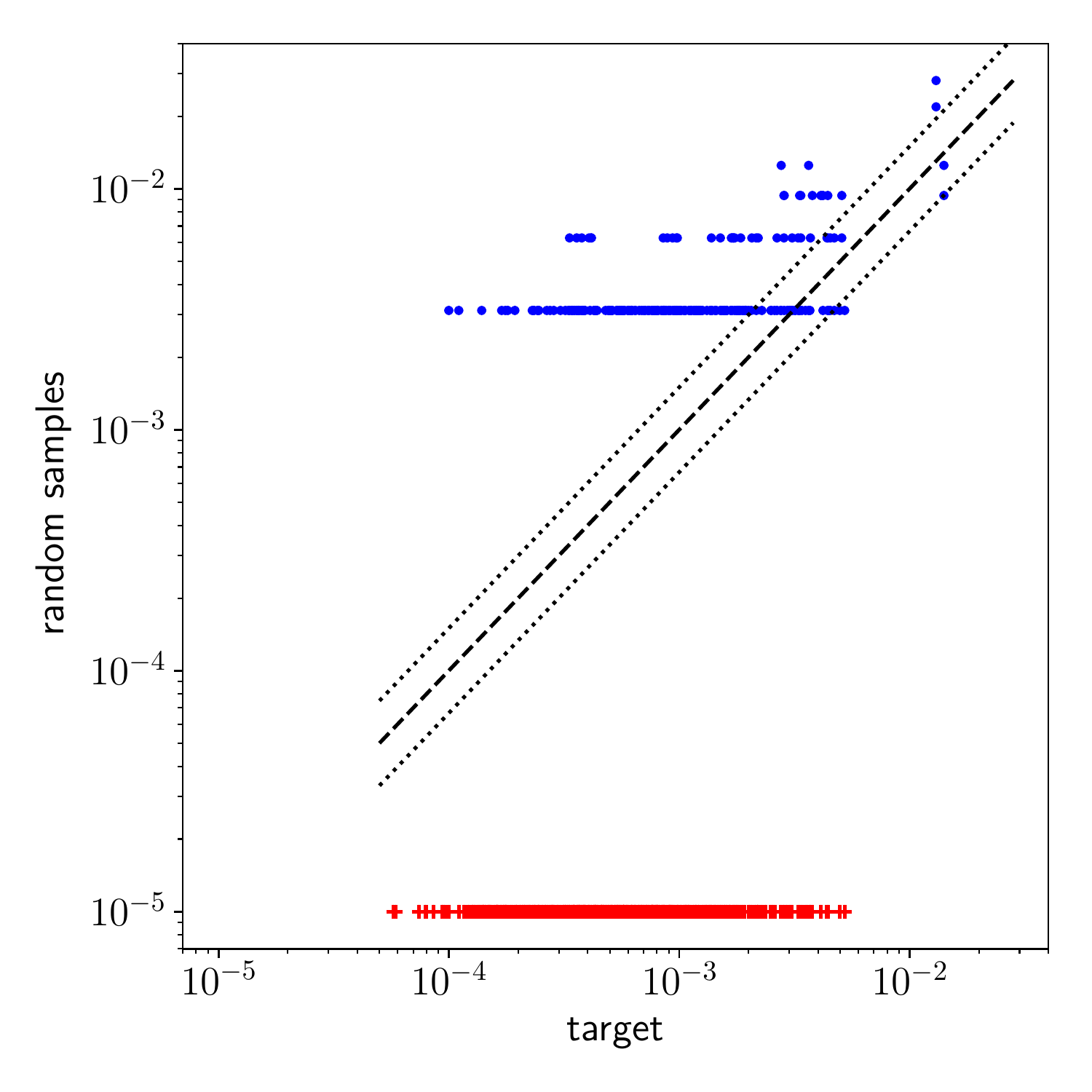} \subcaption{} \label{hbm_pdf_choice} }\end{minipage} 
\begin{minipage}[b ]{\mywidth} {\includegraphics[ width=\textwidth ]{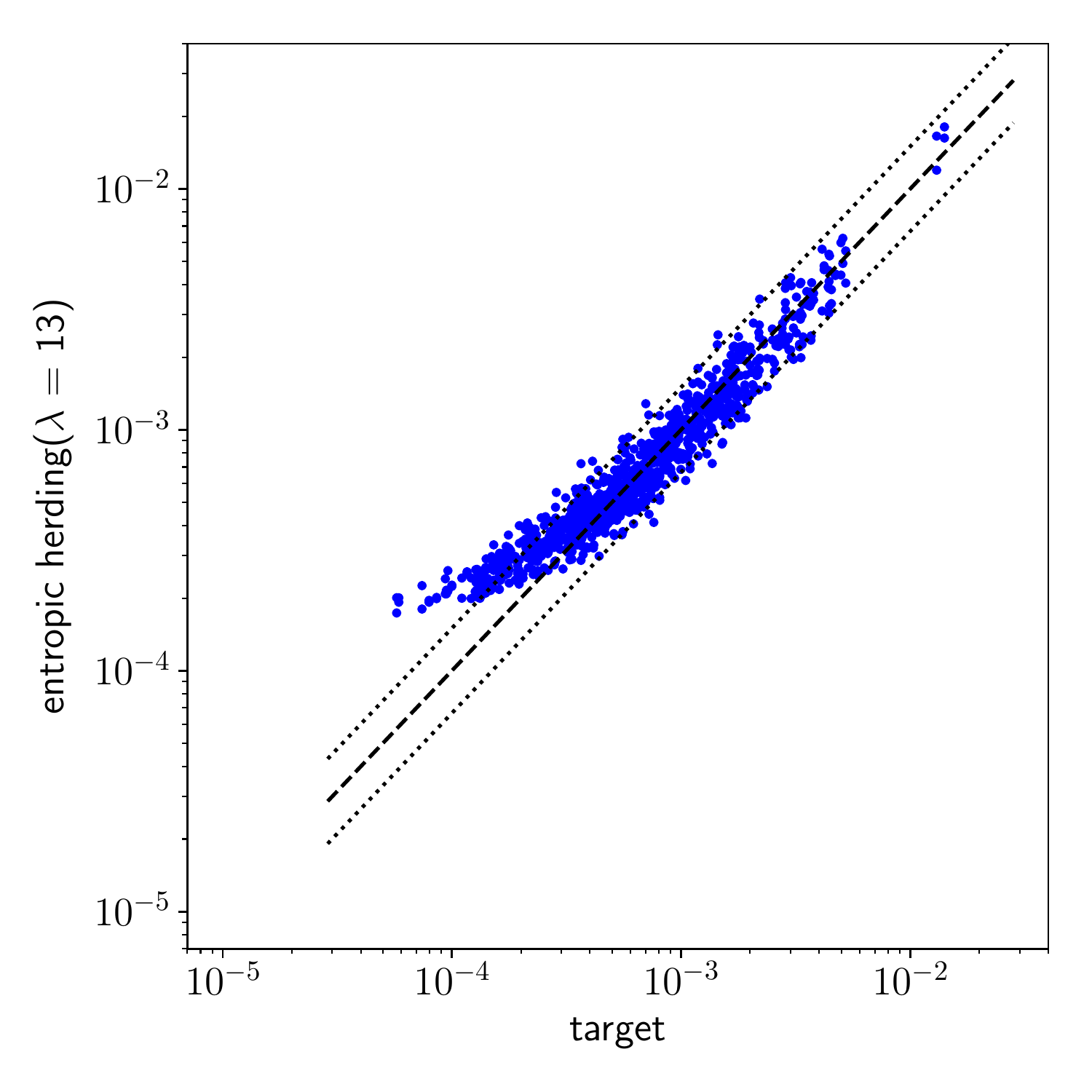} \subcaption{} \label{hbm_pdf_eher} }\end{minipage} 
\end{center}
\caption{Probability mass values of \subref{hbm_pdf_choice} the empirical distribution of random samples and \subref{hbm_pdf_eher} the output distribution of the entropic herding for the Boltzmann machine compared to the target distribution. Each state of the target Boltzmann machine is represented by a point. The states with zero probability mass are represented by a small value of $10^{-5}$ and in red plus signs. The vertical axis represents the probability mass of the obtained distribution. The horizontal axis represents those of the target distribution. A logarithmic scale was used for both axes. The point on the dashed diagonal line indicates that the two probability masses are identical. The point between two dotted lines indicates the difference in mass value is within a factor of 1.5. The result with $\lambda=13$ was chosen by model validation}
\label{hbm_pdf}
\end{lfigure} 

\subsection{  Model selection }

Using the Boltzmann machine above, we present a numerical example of the dependency of parameter choice on output and the model selection for entropic herding. In the experiment, $\lm$ for all $m$ was proportional to the scalar parameter $\lambda$. Figure \ref{hbm_engent} shows the moment error and the entropy of the output distribution for different values of $\lambda$ and number of samples. The error is measured by the sum squared error between the feature mean of the output and the target distribution. The results of the identical number of samples are represented by points connected by black broken lines. By comparing these results, we observe that the error in the feature means mostly decreases with increasing $\lambda$. In contrast, we can obtain a more diverse distribution with a large entropy value by decreasing $\lambda$. Therefore, there is a trade-off between accuracy and diversity when choosing parameter values. 

\begin{lfigure}[ htbp ]
\begin{center}
 
\includegraphics[ width=4.5in ]{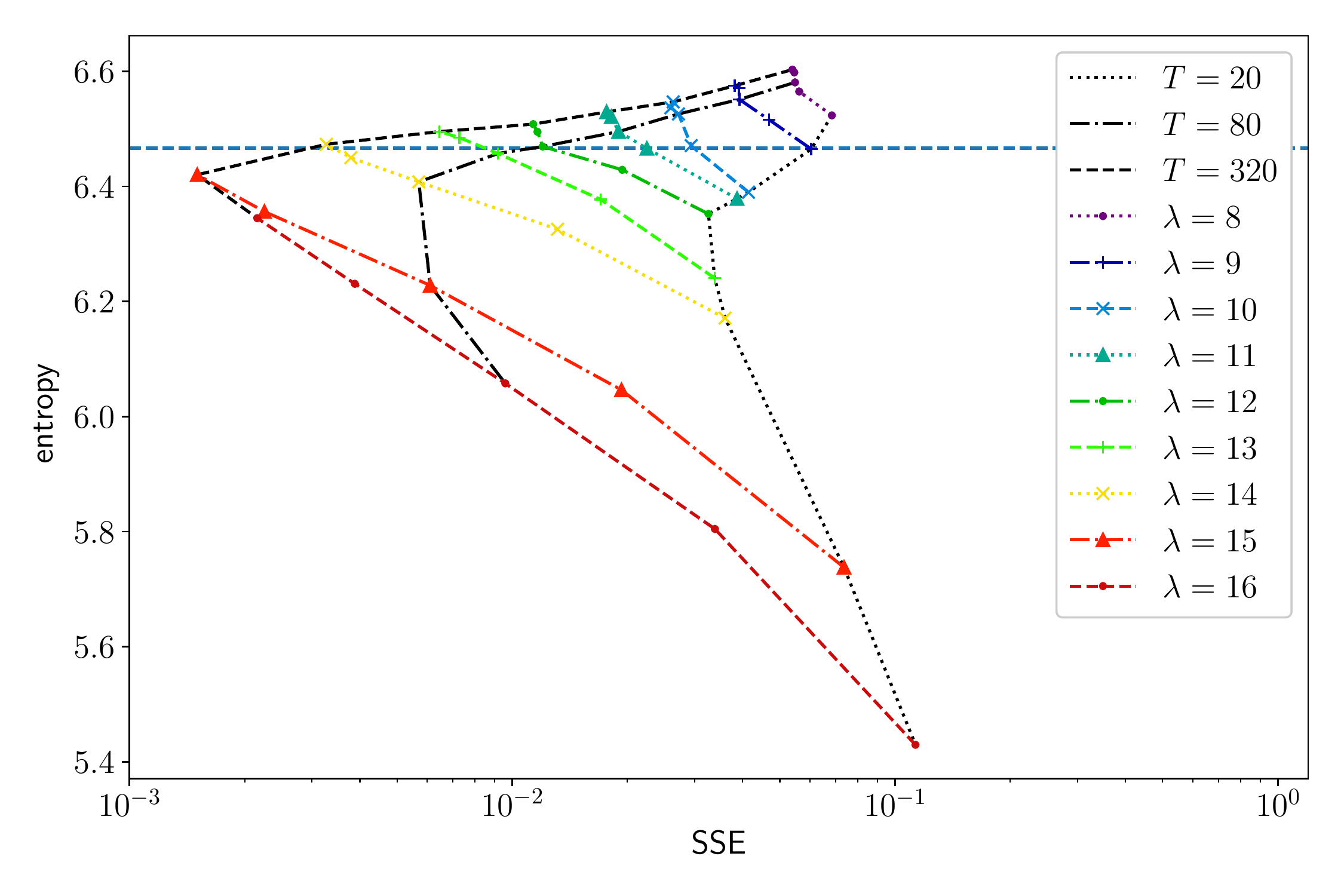} 
\end{center}
\caption{Moment error and the entropy value of the output distribution $\poutput$ for various values of weight parameter $\lambda$ and output length $\toutput\in\{20,40,80,160,320\}$. The mean values of 10 trials are shown. The horizontal axis represents the sum squared error of the feature, which is defined as $SSE=\sum_m(\fm(\poutput)-\mu_m)^2$. It is represented on a logarithmic scale. The vertical axis represents the entropy value, $H(\poutput)$. The horizontal dashed line represents the entropy of target $\pbm$. The results for each $\lambda$ are represented as points with the same color and are connected by lines. The points corresponding to identical $\toutput$ are connected by black lines}
\label{hbm_engent}
\end{lfigure}We can compare the output for various parameters by comparing the KL-divergence between the target distribution $\pbm$ and the output distribution $\poutput$, which is defined as 
\AL{
KL(\pbm\|\poutput)=\sum \pbm(x)\frac{\log \pbm(x)}{\log \poutput(x)},
}
 where the value $\poutput(x)$ can be easily evaluated for each $x$. Note that if we have a validation dataset instead of the target distribution, we can also compare the negative log-likelihood for the validation set. Figure \ref{hbm_kl} shows the KL-divergence for various values of $\lambda$ and the number of samples. We observe that the optimal $\lambda$ is dependent on the number of samples.  

\begin{lfigure}[ htbp ]
\begin{center}
 
\includegraphics[ width=4.5in ]{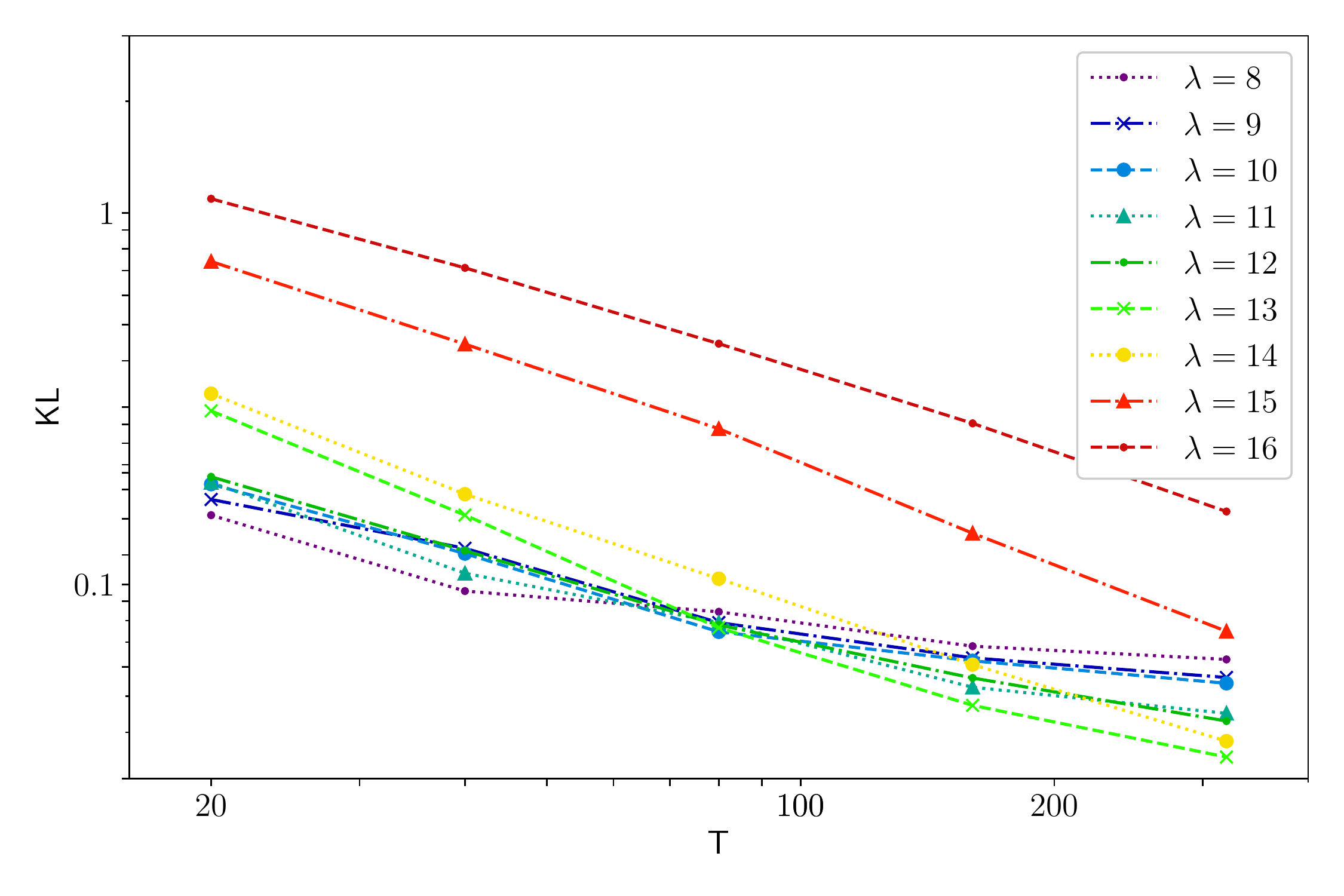} 
\end{center}
\caption{KL-divergence $KL(\pbm\|\poutput)$ for different values of weight parameter $\lambda$ and output length $\toutput\in\{20,40,80,160,320\}$. The mean values of 10 trials are shown. The horizontal and vertical axes represent $\toutput$ and KL-divergence, respectively. The logarithmic scale was used for both axes. The results for each $\lambda$ are represented as points and are connected by lines}
\label{hbm_kl}
\end{lfigure}
\subsection{  UCI wine quality dataset }

Finally, we present an example of an application of entropic herding to real data. We used a wine quality dataset \mycite{WineQuality} from the UCI data repository\footnote{\url{https://archive.ics.uci.edu/ml/datasets/Wine+Quality}, Accessed November 16, 2021}. It was composed of 11 physicochemical features of 4898 wines. The wines are classified into red and white, which have different distributions of feature values. We applied some preprocessing to the data, including log-transformation and z-score standardization. A summary of the features and preprocessing methods is provided in Table \ref{table_wine_features}. 

A simple model for this distribution is a multivariate normal distribution, defined as follows: 
\AL{
p(\vx)=\frac{1}{Z}\exp\AS{-\frac{1}{2}\sum_{i,j}W_{ij}(x_i-\mu_i)(x_j-\mu_j)},
}
 where $\vx=(x_1,\ldots,x_{11})\tr$ is the vector of the feature values, and $Z$ is the normalizing factor. The parameter $W$ in this model can be easily estimated from the covariance matrix of the features. This model is unimodal and has a symmetry such that it is invariant under the transformation $\vx \leftarrow -(\vx-\vec\mu)+\vec\mu$. 

However, as shown in Fig.~\ref{wine_pp_partA} and \ref{wine_pp_partB}, the distribution of this dataset is asymmetric, and bimodal distributions can be observed in the pair plots. To model such a distribution, we improve the model by introducing higher-order terms as follows: 
\AL{
\SP{
p(\vx)&=\frac{1}{Z}\exp\AS{-\frac{1}{2}\sum_{i,j}W_{ij}(x_i-\mu_i)(x_j-\mu_j)
\ASBR{\\&}
\quad - \sum_i \theta^{(3)}_i (x_i-\mu_i)^3 - \sum_i \theta^{(4)}_i (x_i-\mu_i)^4}.
}\LB{eq-wine-target}
}
 The direct parameter estimation for this model is difficult. However, entropic herding can be applied to draw inferences from the feature statistics of the data. We use the following feature set: 
\AL{
\SP{
&\{\phi^{(1)}_i\mid i=1,\ldots,11\} \\
&{} \cup \{\phi^{(2)}_{ij}\mid i,j=1,\ldots,11, i < j\} \\
&{} \cup \{\phi^{(3)}_i\mid i=1,\ldots,11\} \\
&{} \cup \{\phi^{(4)}_i\mid i=1,\ldots,11\},
}
}
 where each feature is defined as 
\AL{
\phi^{(1)}_i(\vx) &= x_i-\mu_i,\\
\phi^{(2)}_{ij}(\vx) &= (x_i-\mu_i)(x_j-\mu_j),\\
\phi^{(3)}_i(\vx) &= (x_i-\mu_i)^3,\\
\phi^{(4)}_i(\vx) &= (x_i-\mu_i)^4.
}
 We added $\phi^{(1)}_i$ to control the mean of each variable. Using the maximum entropy principle with the moment values taken from an assumed background distribution \eq{eq-wine-target}, we reproduce the distribution where the coefficients corresponding to $\phi_i^{(1)}$ are zero. 

We used the candidate distribution set $\qset$, defined as a set of the following distributions: 
\AL{
&x_i \sim \mathcal{N}(\mu_i, \sigma_i^2)\quad^\forall i\in\{1,\ldots,11\},
}
 where $\mu_i\in \Real$ and $\sigma_i > 0.01$ for all $i\in\{1,\ldots,11\}$ are the parameters and $x_1,\ldots, x_{11}$ are independent.

Figure \ref{wine_pp_part} shows the pair plot of the distribution of the dataset and the distribution obtained from entropic herding. We picked three variables in the plot for ease of comparison. The plot for all variables will be presented in Appendix \ref{sec-full-pairplot}. We observe that the distribution obtained well represents the characteristics of the dataset distribution. Particularly, the asymmetry and the two modes are well represented by the output. Figure \ref{wine_circles} illustrates some components in the output distribution $\poutput$. 

\begin{lfigure}[ htbp ]
\begin{center}
 
\begin{minipage}[b ]{\mywidth} {\includegraphics[ width=\textwidth ]{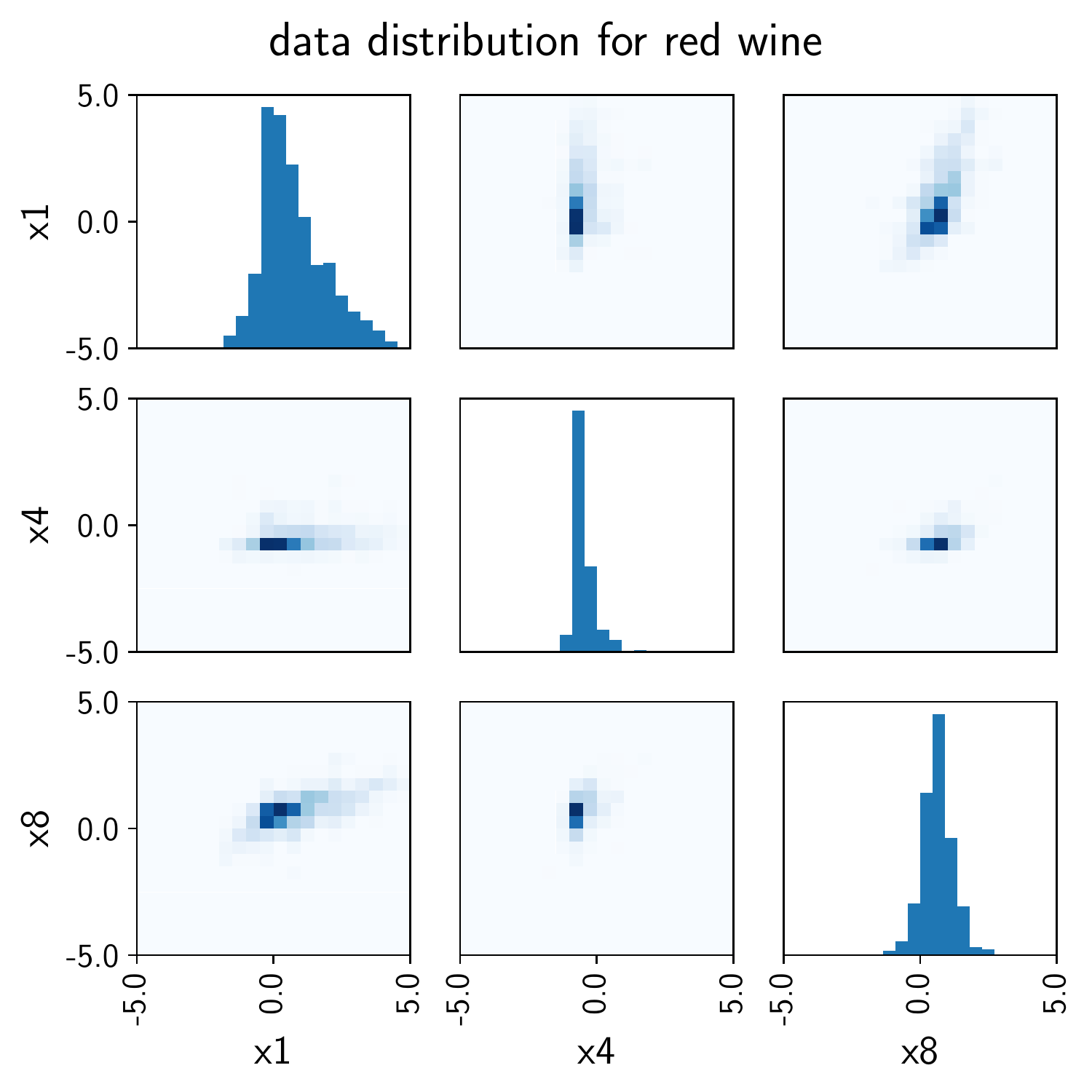} \subcaption{} \label{wine_pp_partA} }\end{minipage} 
\begin{minipage}[b ]{\mywidth} {\includegraphics[ width=\textwidth ]{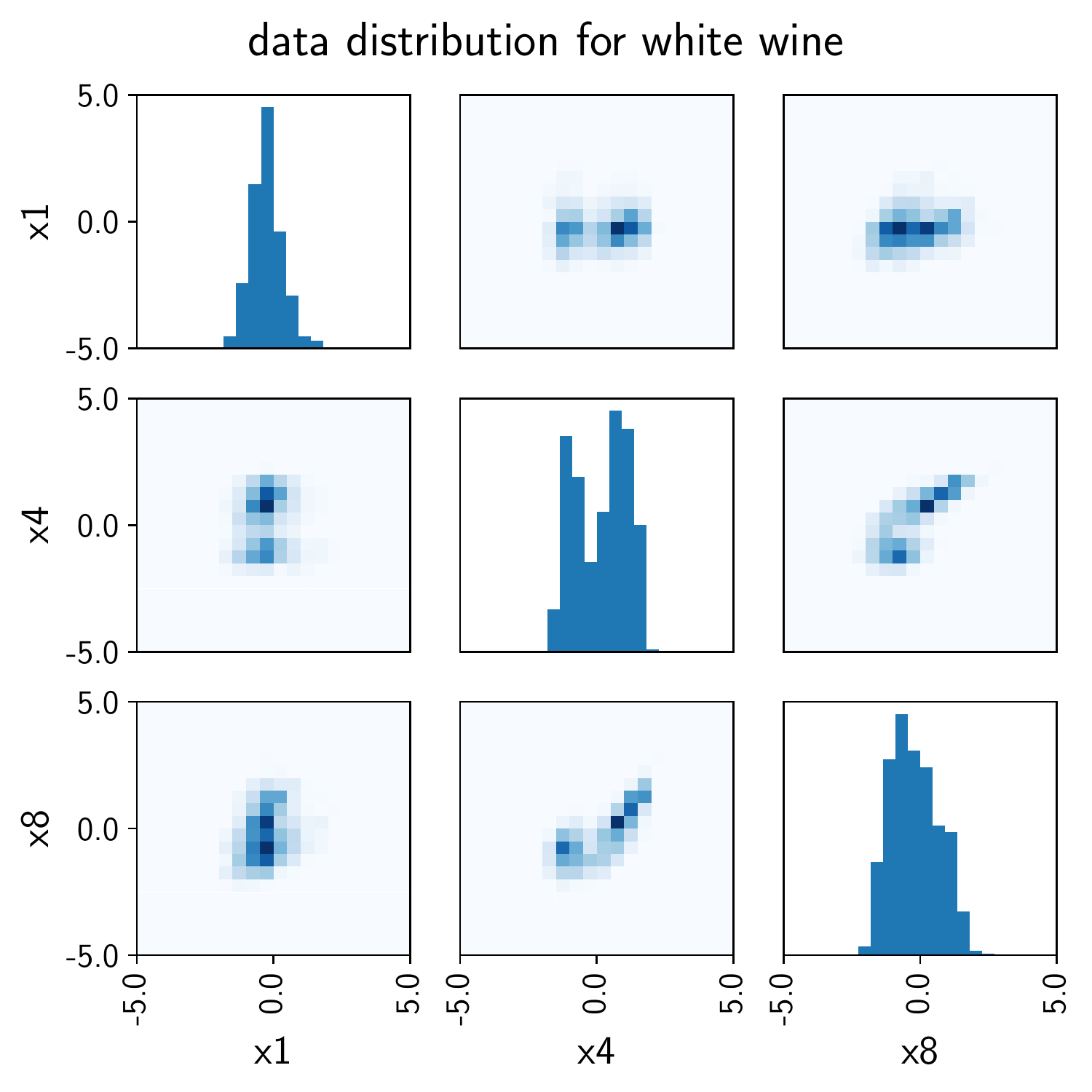} \subcaption{} \label{wine_pp_partB} }\end{minipage} \\ 
\begin{minipage}[b ]{\mywidth} {\includegraphics[ width=\textwidth ]{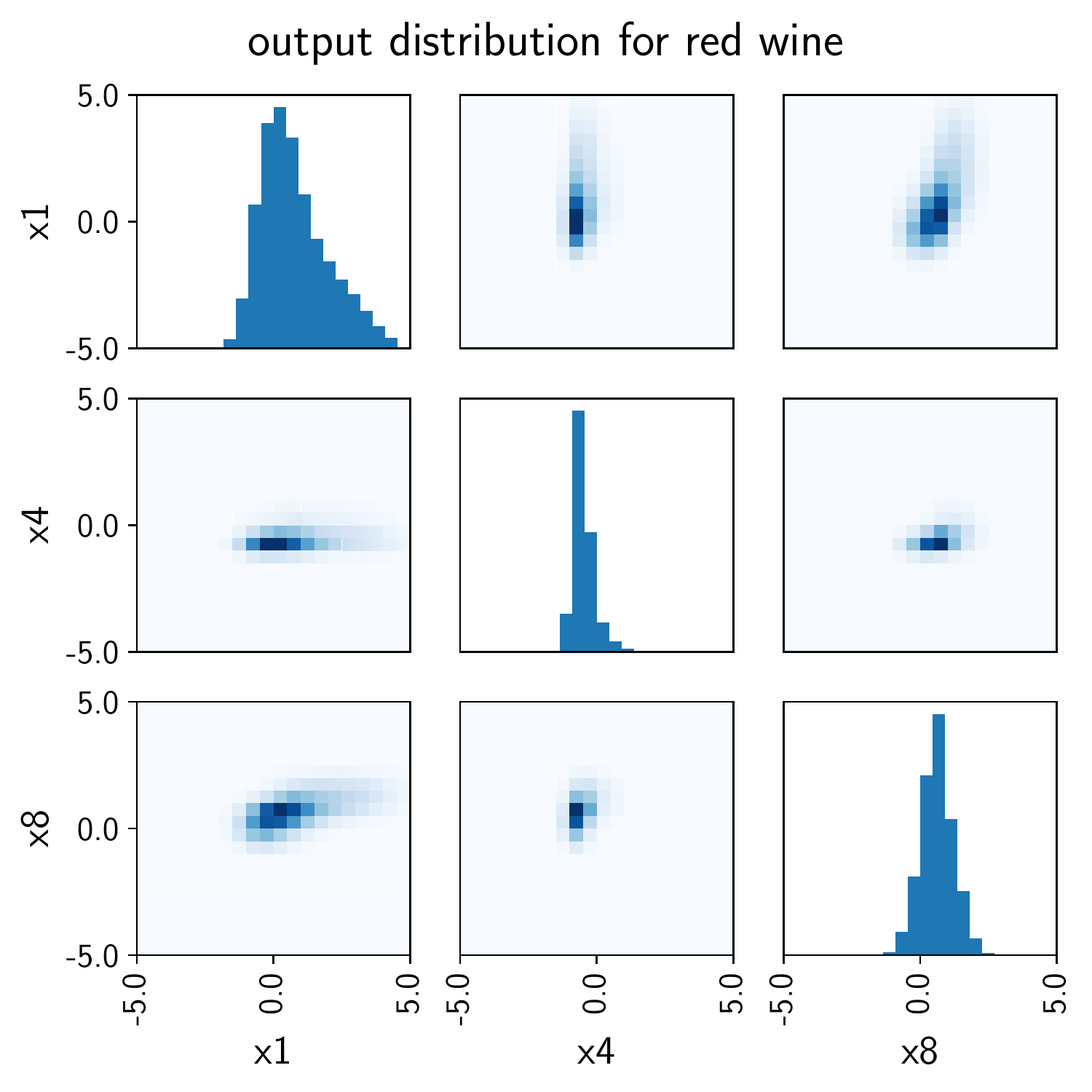} \subcaption{} \label{wine_pp_partC} }\end{minipage} 
\begin{minipage}[b ]{\mywidth} {\includegraphics[ width=\textwidth ]{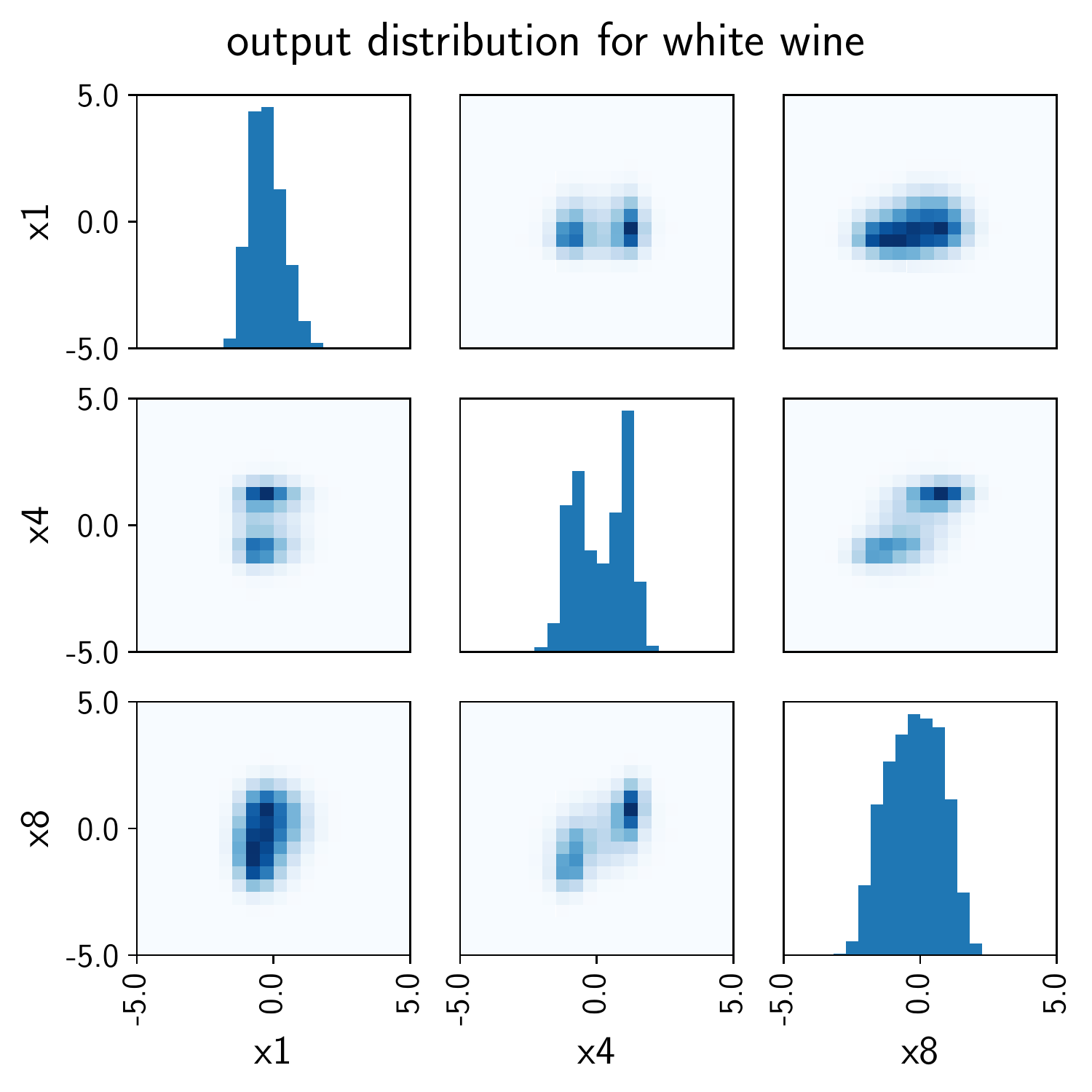} \subcaption{} \label{wine_pp_partD} }\end{minipage} 
\end{center}
\caption{Pair plot of the distribution of the dataset and the distribution obtained from entropic herding. For ease of comparison, we selected three variables $(x_1, x_4, x_8)$, which are shown from top to bottom and from left to right in each panel. The plots for all variables are presented in the Appendix. The plots on the diagonal represent the histograms of the variables. The vertical axis has its own scale of probability mass and does not correspond to the scale shown in the plot. The other plots represent the probability density with cells of width 0.5. The horizontal and vertical axes represent the variables corresponding to rows and columns, respectively. The left (\subref{wine_pp_partA} and \subref{wine_pp_partC}) and right (\subref{wine_pp_partB} and \subref{wine_pp_partD}) panels correspond to the red and white wine, respectively. The top (\subref{wine_pp_partA} and \subref{wine_pp_partB}) and bottom (\subref{wine_pp_partC} and \subref{wine_pp_partD}) panels are for the dataset and entropic herding, respectively}
\label{wine_pp_part}
\end{lfigure} 

\begin{sfigure}[ htbp ]
\begin{center}
 
\includegraphics[ width=\mywidth ]{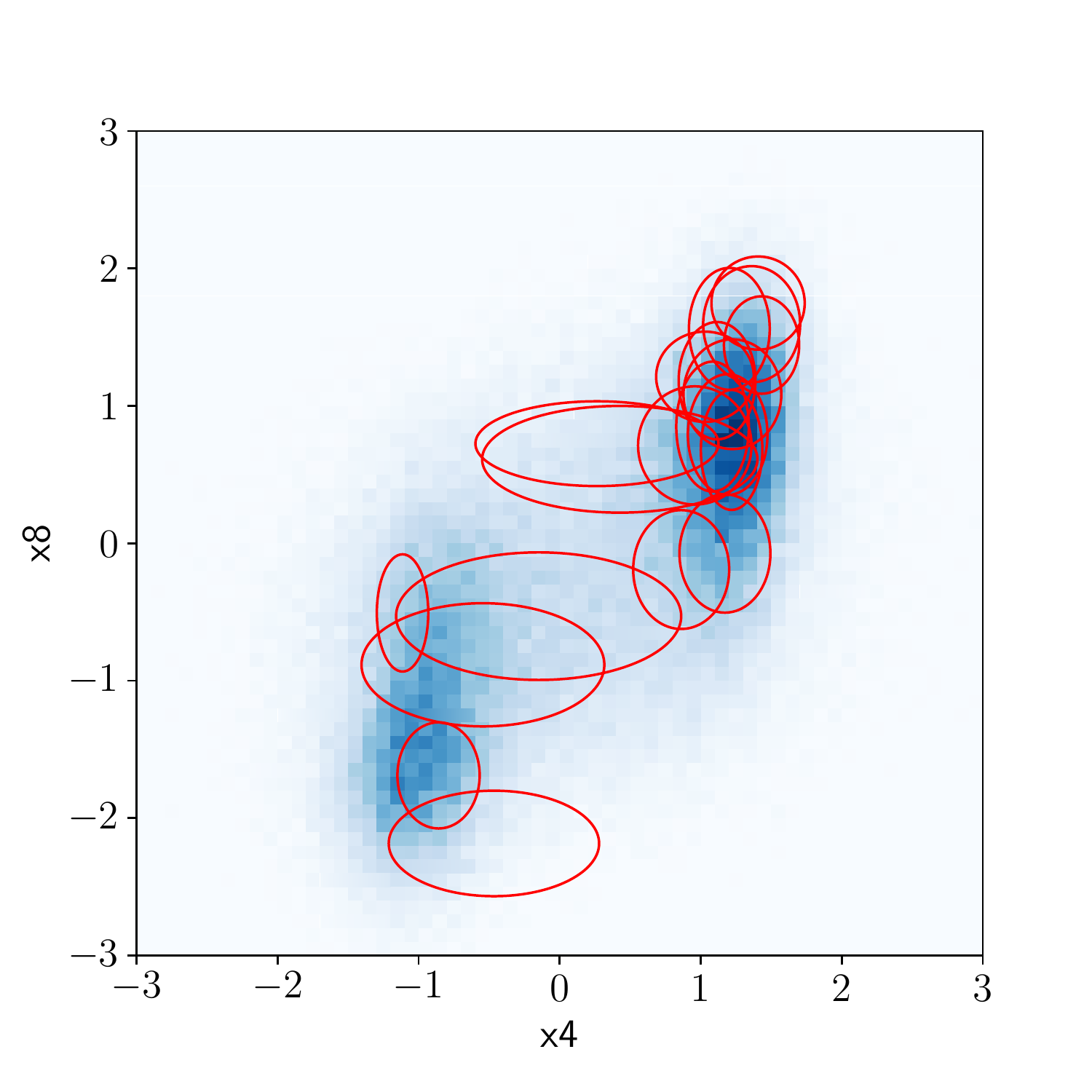} 
\end{center}
\caption{Distribution of $x_4$ (horizontal axis) and $x_8$ (vertical axis) obtained by entropic herding for white wine. The width of each cell in the plot is 0.1. The red circles represent 20 distributions $r\tt$ randomly drawn from the output. The sizes of the circles along the vertical and horizontal axes represent the standard deviations of $x_4$ and $x_8$, respectively}
\label{wine_circles}
\end{sfigure}We can use the {\herding} output as a probabilistic model. Figure \ref{wine_nll} shows the negative log-likelihood of the validation data. We observe that the model corresponding to the true class of wine assigns larger likelihood values than the other models. We used the results for the classification of red and white wine by using the difference in the log-likelihood as a score. The AUC for the validation set was 0.998. The score was close to the AUC value obtained using the log-likelihood of fitted multivariate normal distribution (0.995) and linear logistic regression (0.998).  

\begin{lfigure}[ htbp ]
\begin{center}
 
\includegraphics[ width=4.5in ]{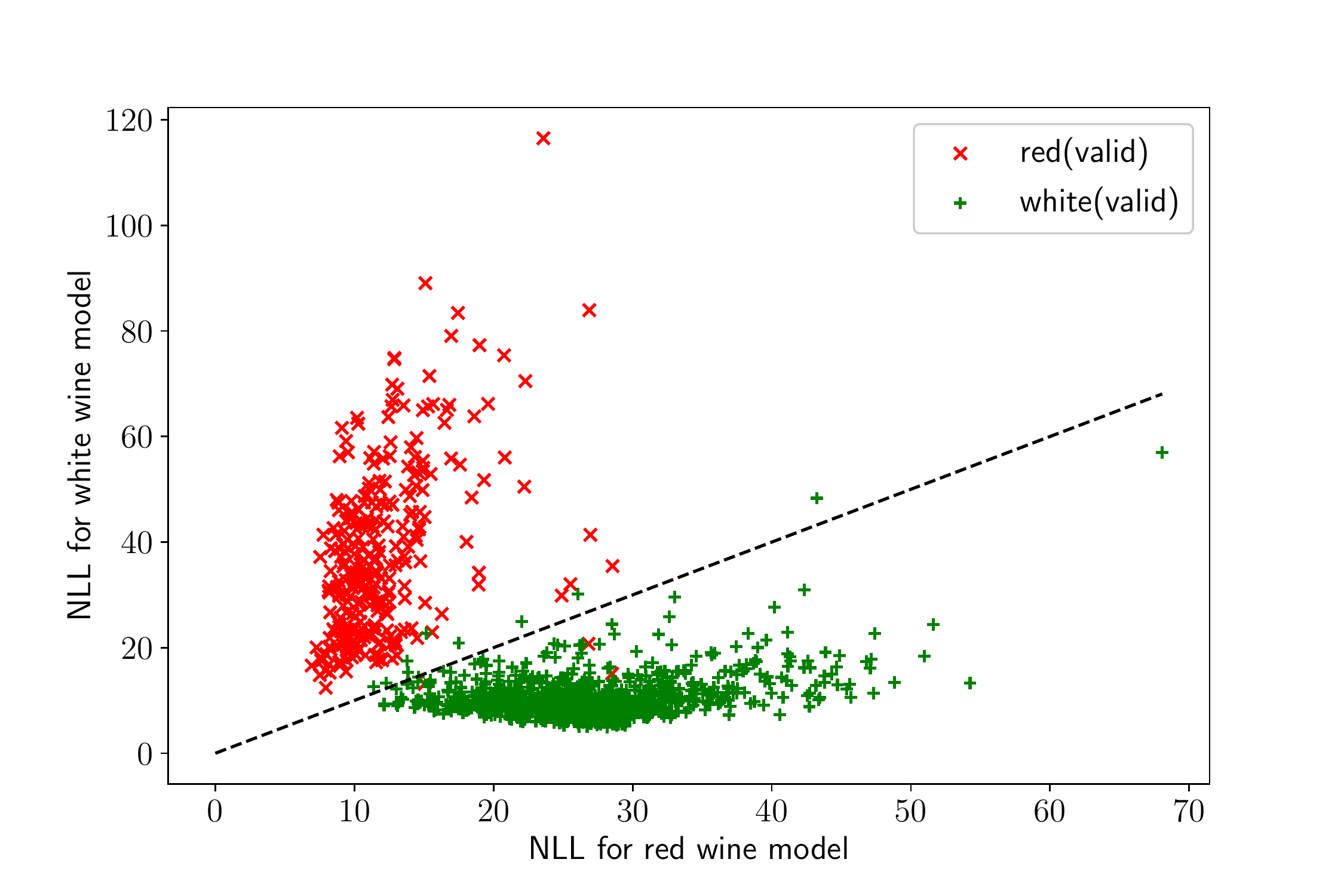} 
\end{center}
\caption{Negative log-likelihood for the validation data. The red crossed marks and green plus signs represent the red and white wines in the validation data, respectively. The horizontal and vertical axes represent the negative log-likelihood $-\log \poutput(\vx)$ for models obtained by entropic herding for red and white wines, respectively. The dashed line represents where the two negative log-likelihoods are identical}
\label{wine_nll}
\end{lfigure}The simple analytic form of the output distribution can also be used for the probabilistic estimation of missing values. We generated a dataset with missing values by dropping $x_4$ from the validation set for white wine. The output distribution is $\poutput(x)=\frac{1}{\tmax}\sum_{T=1}^{\tmax} r\tt(x)$, where $r\tt\in \qset$ is given by the parameter $(\mu_i\tt, \sigma_i\tt)$ for $i=1,\ldots,11$. Let $r_i(x_i;\mu_i\tt,\sigma_i\tt)$ denote the marginal distribution of $x_i$ for $r\tt$, which is a normal distribution. The conditional distribution of $x_4$ on the other variables is expressed as follows: 
\AL{
\SP{
    &\poutput(x_4\mid \{x_i \mid i\neq 4\})\\
    &{} = \frac{1}{Z}\sum_{T=1}^{\tmax}w\tt r_4(x_4;\mu_4\tt,\sigma_4\tt),
}
}
 where $w\tt=\prod_{i\neq 4}r_i(x_i;\mu_i\tt,\sigma_i\tt)$ and $Z=\sum_{T=1}^{\tmax}w\tt$, respectively. Figure~\ref{wine_violinA} shows a violin plot of the conditional distribution for 50 randomly sampled data. Figure~\ref{wine_violinB} shows the results of the multivariate normal distribution. The standard deviations of the estimations shown in Fig.~\ref{wine_violinB} are identical because they are from the same multivariate normal distribution. Comparing these plots, we see that entropic herding is better for the more flexible model than the multivariate normal distribution. The dotted horizontal line shows the true value, and the short horizontal lines show the $[10, 90]$ quantile of the estimated distribution. We counted the number of data with true values in this range. The proportion of such data was 79.7\% for entropic herding and 51.1\% for multivariate normal distribution. We can conclude that estimation by entropic herding was better calibrated. 

\begin{lfigure}[ htbp ]
\begin{center}
 
\begin{minipage}[b ]{6in} {\includegraphics[ width=\textwidth ]{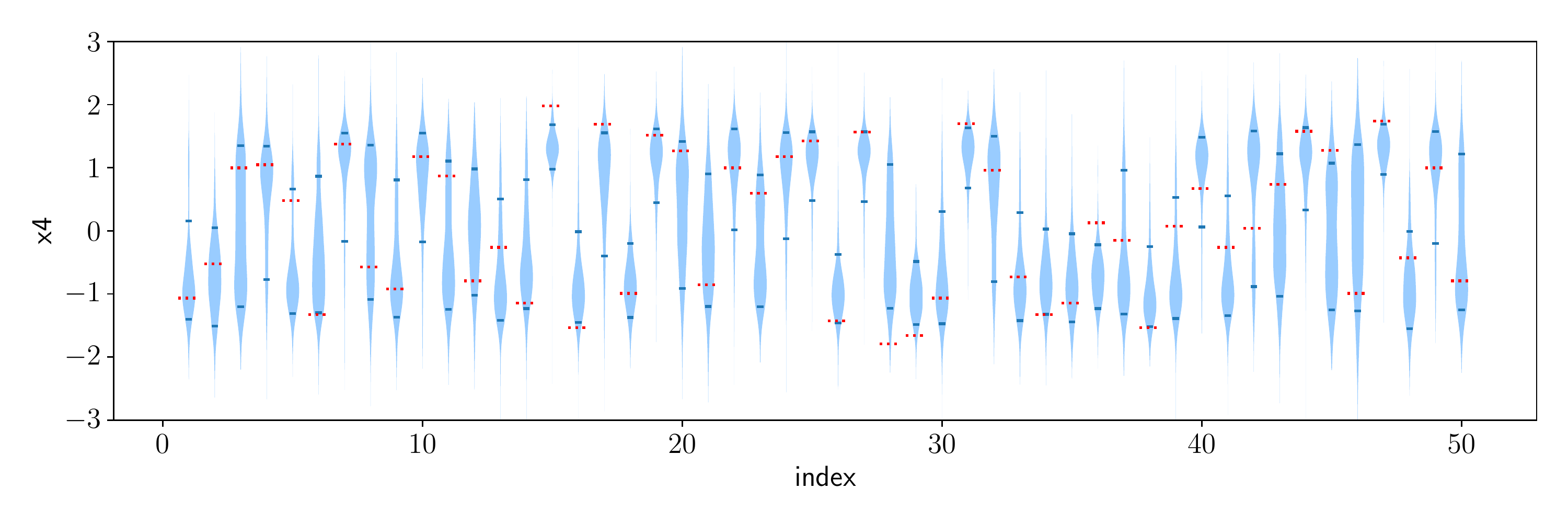} \subcaption{} \label{wine_violinA} }\end{minipage} \\ 
\begin{minipage}[b ]{6in} {\includegraphics[ width=\textwidth ]{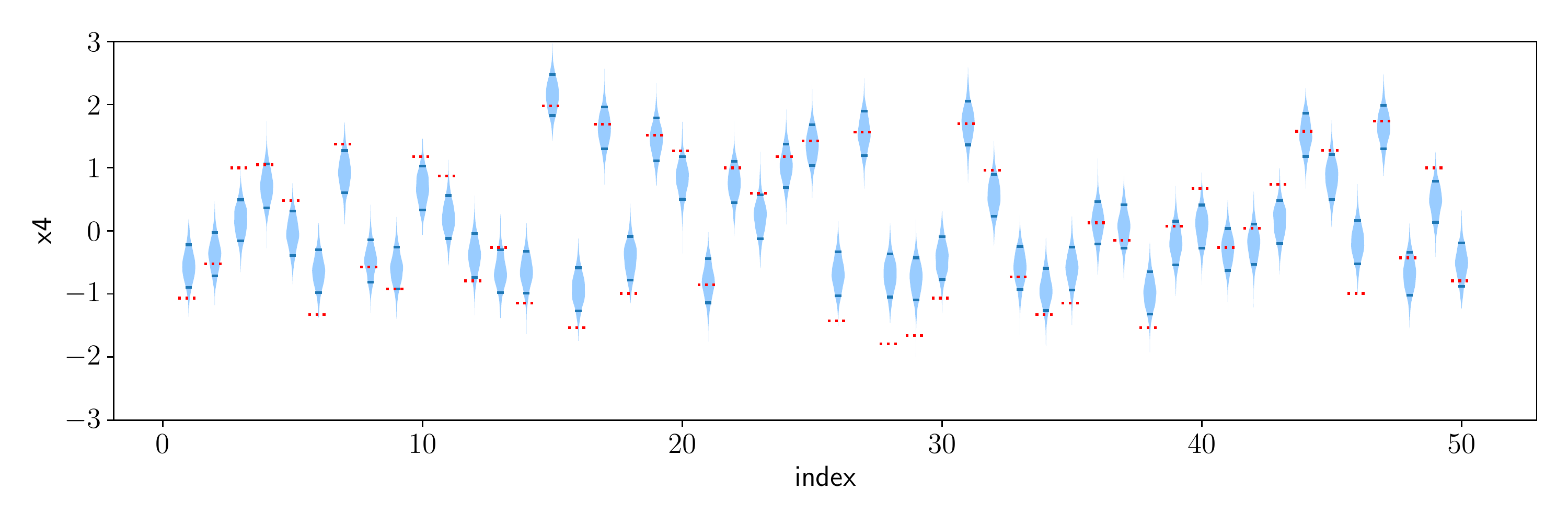} \subcaption{} \label{wine_violinB} }\end{minipage} 
\end{center}
\caption{The violin plots of the conditional distribution of $x_4$ given other variables for 50 randomly sampled validation data. The horizontal axis represents the data index. The vertical axis corresponds to the value of $x_4$. The dotted horizontal line represents the true values in the data. The short horizontal lines show the $[10, 90]$ quantile of the estimated distribution. \subref{wine_violinA} The output distribution of entropic herding. \subref{wine_violinB} The multivariate normal distribution}
\label{wine_violin}
\end{lfigure}

\begin{table*}[htbp]\begin{center}
\caption{Summary of the UCI wine quality dataset and the preprocessing applied. The columns ``variable'' and ``name'' show the correspondence between the indices and names of the features. The column ``range'' shows the minimum and maximum values in the dataset prior to preprocessing. The rightmost two columns represent the applied preprocessing. The checkmark (\cmark) indicates that preprocessing in the column is applied. ``Log-transformation'' means the transformation $x\leftarrow \log_{10} x$. The checkmark with an asterisk (\cmark*) indicates that there exist data such that $\log_{10}x\to -\infty$. We assigned $-5.0$ instead to these cases. Linear transformation is applied to each variable to make the average zero and the variance one, as represented in the column ``z-score''}
\begin{tabular}{rlrcc}
\multicolumn{1}{l}{variable}&name&\multicolumn{1}{l}{range}&\multicolumn{1}{l}{log-transformation}&\multicolumn{1}{l}{z-score}\\\hline
$x_1$&fixed acidity&$[ 3.80 , 15.90 ]$&\xmark&\cmark\\
$x_2$&volatile acidity&$[ 0.08 , 1.58 ]$&\xmark&\cmark\\
$x_3$&citric acid&$[ 0.00 , 1.66 ]$&\cmark*&\cmark\\
$x_4$&residual sugar&$[ 0.60 , 65.80 ]$&\cmark&\cmark\\
$x_5$&chlorides&$[ 0.01 , 0.61 ]$&\cmark&\cmark\\
$x_6$&free sulfur dioxide&$[ 1.00 , 289.00 ]$&\cmark&\cmark\\
$x_7$&total sulfur dioxide&$[ 6.00 , 440.00 ]$&\cmark&\cmark\\
$x_8$&density&$[ 0.99 , 1.04 ]$&\xmark&\cmark\\
$x_9$&pH&$[ 2.72 , 4.01 ]$&\xmark&\cmark\\
$x_{10}$&sulphates&$[ 0.22 , 2.00 ]$&\xmark&\cmark\\
$x_{11}$&alcohol&$[ 8.00 , 14.90 ]$&\xmark&\cmark\\
\end{tabular}
\label{table_wine_features}
\end{center}\end{table*}

\section{ Discussion  }
\label{sec-discussion}
The most significant difference between proposed entropic herding and original point herding is that entropic herding represents the output distribution as a mixture of probability distributions. As for the applications of entropic herding, some of the desirable properties of density calculation and sampling, discussed in Section \ref{sec-adv}, are the results of using the distribution mixture for the output. There are also many probabilistic modeling methods that use the distribution mixture, such as the Gaussian mixture model and kernel density estimation \mycite{ParzenKDE}. All of these methods, including entropic herding, share the above characteristics.

The most important difference between entropic herding and other methods is that it does not require specific data points and only uses the aggregated moment information of the features. The use of aggregated information has recently attracted increasing attention \mycite{SheldonNIPS11,LawNIPS18,TanakaAAAI19,TanakaNIPS19,ZhangNIPS20}. Sometimes, we can only use the aggregated information for privacy reasons. For example, statistics, such as population density or traffic volumes, are often aggregated to mean values by spatial regions, which often have various granularities \mycite{LawNIPS18,TanakaAAAI19,TanakaNIPS19}. In addition to data availability, features can be selected to avoid irrelevant information depending on the focus of the study and data quality. These advantages are common to entropic and point herding methods, but nonetheless distinctive when compared with other probabilistic modeling methods. Notably, kernel herding \mycite{KernelHerding} is a prominent variant of herding that has a convergence guarantee, but it does not share the aforementioned advantages because it requires individual data points to use the features defined in the reproducing kernel Hilbert space.

The framework of entropic herding does not depend on the choice of the feature functions $\pm$ and candidate distribution $\qset$. In practice, the requirement for $\qset$ is the availability of the calculation and optimization of the expectation $\expe{q}{\pm(x)}$ and the entropy $H(q)$ for $q\in\qset$. In the numerical examples, we used the distribution of independent random variables for $\qset$ for simplicity. In recent years, many methods have been developed for the generative expression of a probability distribution, for example, using neural networks \mycite{GAN14,VAE14} and decision trees \mycite{DensityForest}. We expect that this study will serve as a theoretical framework for the more advanced use of these sophisticated generative models which can use them as a mixture.

Regarding computational efficiency, the most computationally intensive part of entropic herding at present is the optimization step. We must repeatedly solve the optimization problems generated from the weight dynamics. If the amount of weight update in each step is small, we can assume that the problems in the consecutive steps are similar. As described in the Appendix, we used gradient descent from the latest solution for the optimization of the experiment and demonstrated its feasibility. However, exploiting the characteristics of repetitive optimization will produce further improvement.

The gradient descent in entropic herding is easily realizable if the analytic form of entropy and the expectation of the feature over the candidate distribution $q\in \qset$ are available. However, if the analytic form is not available, we must resort to more sophisticated optimization methods. The problem solved in the optimization step is equivalent to obtaining the distribution approximation by minimizing the KL-divergence (see Eq.~(\ref{eq-opt-kl})). This problem often appears in the field of machine learning, such as in variational Bayes inference. To extend the application of {\herding}, it can be combined with recent optimization techniques developed in this field. 

\section{ Conclusion  }
\label{sec-conclusion}
In this paper, we proposed an algorithm called entropic herding as an extension of {\herding}. 

By using the proposed algorithm as a framework, we discussed the connection between {\herding} and the maximum entropy principle. Specifically, entropic herding is based on the minimization of the target function $\losa$. This function, which is composed of the feature moment error and entropy term, represents the maximum entropy principle. Herding minimizes this function in two ways. The first is the minimization of the upper bound of the target function by solving the optimization problem in each step. The second is the diversification of the samples by the complex dynamics of the high-dimensional weight vector. We also studied the output of entropic herding through a mathematical analysis of the optimal distribution of this minimization problem.    

We also clarified the difference between entropic and point herding for application both theoretically and numerically. We demonstrated that point herding can be extended by explicitly considering entropy maximization in the optimization step by using distributions rather than points for the candidates. The output sequence of the entropic herding has more efficiency in the number of samples than the point herding because each output is a distribution that can assign a probability mass to many states. The output sequence of the distribution can be used as a mixture distribution that allows independent sample generation. The mixture distribution also has an analytic form. Therefore, model validation using likelihood and inference through conditional distribution is also possible. 

As discussed in Section \ref{sec-discussion}, entropic herding allows flexibility in the choice of the feature set and candidate distribution. We expect entropic herding to be used as a framework for developing effective algorithms based on the distribution mixture.

\bmhead{Acknowledgments}
This work is partially supported by JST CREST (JP-MJCR18K2), by AMED (JP21dm0307009), by UTokyo Center for Integrative Science of Human Behavior (CiSHuB), by the International Research Center for Neurointelligence (WPI-IRCN) at The University of Tokyo Institutes for Advanced Study (UTIAS), and by JST Moonshot R\&D Grant Number JPMJMS2021.


\begin{appendices}
 \appendix 

\section{ Details of optimization algorithms and experiments  }
\label{sec-implementation}
Here, we present detailed descriptions of the methods used for the experiments in this study.

The preprocessing applied to the feature functions is detailed below. Let $\phi_m$ be the $m$th feature function provided to the algorithm. Let $x_1,\ldots, x_D$ be the input data vector. We calculate the mean $\mu_m$ and standard deviation $\sigma_m$ of the feature values, namely, $\phi_m(x_1),\ldots,\phi_m(x_D)$. We then standardize the feature values as follows: 
\AL{
\phi'_m(x) = \frac{1}{\sigma_m}(\phi_m(x)-\mu_m).
}
 The feature value for the distribution is defined as $\fm'(p)\equiv\expe{p}{\phi'_m(x)}$. After standardization, the average of the feature values used for the target value becomes zero, and we use the same $\Lambda_m$ for all $m$. Namely, we used  
\AL{
\am\tt=\lambda\fm'(p\tt)
}
 instead of using \eq{eq-am}. Note that this is equivalent to $\am\tt=\frac{\lambda}{\sigma_m}(\fm(p\tt)-\mm)$, which can also be obtained by substituting $\Lambda_m=\lambda/\sigma_m$ into \eq{eq-am}. Here, the parameter tuning of $\Lambda_m$ is simplified to optimize the single global parameter $\lambda$. Note that the preprocessing simplifies parameter tuning and is not generally necessary. When we do not have information on the standard deviation, we can still use entropic herding by tuning each $\lm$. 

The optimization problem in Eq.~(\ref{eq-opt-entropic}) was solved using the gradient method. A different parameterized candidate distribution set $\qset$ is used for each case, and the parameter is optimized by repeating a small update following the gradient of the target function. The optimization for each $T$ is performed by iterating the number $\kupdate$ of optimization steps. The obtained state $r\tt$ is used as the initial state for the next optimization at $T+1$. The amount of update in each optimization step was modified using the Adam method \mycite{Adam} to maintain the numerical stability of the procedure. The hyperparameters in Adam were set to $(\beta_1, \beta_2, \eps) = (0.8, 0.99, 10^{-8})$ (see \mycite{Adam}). The gradient after the modification was multiplied by the learning rate, denoted by $\etalearn$. At the beginning of the inner loop of optimization, for each $T$, the rolling means maintained by Adam were reset to the initial value of zero. 

In some cases, we used modified weight values depending on the optimization state. Instead of $\am\tp$, we used the weight values $\am'$ defined as follows: 
\AL{
\am'&\equiv\am\tp+\ups\tt\lm(\fm(\qcur)-\fm(p\tp))\LN{eq-am2}
\SP{
&=\am\tp\\
&\qquad+\ups\tt\AS{\lm(\fm(\qcur)-\mm)-\am\tp},
}
}
 where $\qcur$ denotes the current state in the inner loop of the optimization. Note that this is different from $r\tp$, except for the beginning of the inner loop. This modification was also used for numerical stability. The justification is as follows: Eq.~(\ref{eq-tloa-decomp1}) is equivalent to 
\AL{
\SP{
    \tloa\tt&=\Biggl(\summ\frac{1}{2\lm}\biggl(\am'\biggr.\Biggr.\\
    &\Biggl.\biggl.\qquad{}+\ups\tt\lm(\fm(r\tt)-\fm(\qcur))\biggr)^2\Biggr)
}\NN
&\quad- \AS{\ten\tp+\ups\tt(\ent(r\tt)-\ten\tp)}.
}
 The terms dependent on $r\tt$ are  
\AL{
    &\ups\tt\AS{\AS{\summ\am'\fm(r\tt)}-\ent(r\tt)}\NN
    &{}+O((\ups\tt)^2).
}
 By neglecting small higher order term $O((\ups\tt)^2)$, the minimization is reduced to solving $r\tt=\argmin_{q\in\qset}\AS{\AS{\summ \am'\fm(q)}-H(q)}$, which is equivalent to Eq.~(\ref{eq-opt-entropic}) with the substitution of $\am$ by $\am'$.

We introduced a stochastic jump in the optimization step in some cases. In this case, a jump is proposed with a probability of $\pjump$ in each optimization step. The candidate distribution $q'\in \qset$ is drawn randomly and accepted as the next state if it has a better target function value than the distribution of the current state. The method of candidate generation is described for each problem in the following sections. 

The amount of the update in each step, denoted by $\ups\tt$, is set to the same value. Namely, we set $\ups\tt=\epsherding$ for each $T$. This means that the distribution weights $\rst$ in Section \ref{sec-greedy} geometrically decay at the rate of $\rho=1-\epsherding$.

To eliminate nonstationary behavior depending on the initial condition, we set a burn-in period in the algorithm similarly to conventional MCMC algorithms. After the {\herding} run with $\tmax=\tburnin+\toutput$, the output sequence, except for the burn-in period, is aggregated into an output mixture distribution. The output mixture is obtained as follows: 
\AL{
\poutput(x)&=\frac{1}{\toutput}\sum_{t=\tburnin+1}^{\tburnin+\toutput}r^{(t)}(x).
}

We implemented the method above using the automatic differentiation provided by the Theano \mycite{theano} framework. The default settings for the above parameters are summarized in Table \ref{table_params_in_algo}. The values in the table were used unless explicitly stated otherwise.

\begin{table*}[htbp]\begin{center}
\caption{Default settings for each experiment in this paper. A description of these parameters is in the text. The check mark ($\cmark$) in column $\am'$ indicates that the modified weight value (Eq.~(\ref{eq-am2})) is used.}
\begin{tabular}{lrrrrrcr}
&$\epsherding$&$\toutput$&$\tburnin$&$\etalearn$&$\kupdate$&$\am'$&$\pjump$\\\hline
bimodal distribution&0.02&100&50&0.2&50&\cmark&0\\
bimodal distribution (point herding)&0.002&1000&500&0.2&50&\cmark&0\\
Boltzmann machine&0.05&320&100&0.2&50&-&0.1\\
UCI wine quality data&0.01&500&100&0.2&20&\cmark&0.1\\
\end{tabular}
\label{table_params_in_algo}
\end{center}\end{table*}

\subsection{  One-dimensional bimodal distribution }

Using the Metropolis--Hastings method, 10000 samples were generated and used as the input. We used the candidate distribution set $\qset$ defined as 
\AL{
\qset=\{\mathcal{N}(\mu,\sigma^2)\mid \mu\in\Real, \sigma > 0.01\}.
}

The four feature means over the candidate distribution, denoted by $\et_i(q)=\expe{q}{\phi_i(x)}=\expe{q}{x^i}$, are obtained as follows: 
\AL{
\et_1(q) =\expe{q}{x}&=\mu,\\
\et_2(q) =\expe{q}{x^2}&=\mu^2+\sigma^2,\\
\et_3(q) =\expe{q}{x^3}&=\mu^3+3\mu\sigma^2,\\
\et_4(q) =\expe{q}{x^4}&=\mu^4+6\mu^2\sigma^2+3\sigma^4.
}
 As they all have analytic expressions, the gradients with respect to the parameters can be easily obtained.

We applied variable transformation $l=\log \sigma$ and optimized the parameter set $(\mu,l)\in \Real\times[\log 0.01,+\infty)$ in the algorithm. We reported the results for $\lambda=100$ in this study.

For the case of entropic herding with stochastic update, $\pjump=0.1$ is used. When a random jump is proposed, the candidate distribution is determined by using the current value of $l$ and drawing $\mu$ randomly from $[\mu_{min}, \mu_{max}]$, where $\mu_{min}$ and $\mu_{max}$ are the minimum and maximum values of $\mu$ that have so far appeared in the procedure. 

Point herding with Metropolis updates was implemented for comparison. In this case, a random jump is proposed in each step ($\pjump=1$). The candidate is accepted according to the Metropolis rule. That is, it is accepted with a probability $\min(1, \exp(-\Delta F))$, where $\Delta F$ is the increase in the target function $\loss_{\va\tp}(x)$. In this case, the modified weight values $\am'$ are not used. 

\subsection{  Boltzmann machine }

The parameters $W_{ij}$ ($i<j$) for $\pbm$ are drawn from a normal distribution with mean zero and variance $(0.2)^2/N$. We then added some structural interactions to increase nontrivial correlations; we assigned $W_{45}=0$ and $W_{i(i+1)}=-0.3$ for $i\neq 4$. We obtained the values of $\mu_m$ and $\sigma_m$ used for feature standardization by calculating the expectation following the definition of the target model $\pbm$.

We used the candidate distribution set $\qset$, defined as consisting of the following set of random variables: 
\AL{
\SP{
&x_i = \left\{\begin{array}{l}
+1 \quad (\text{with probability } p_i)\\
-1 \quad (\text{with probability } 1-p_i)\\
\end{array}\right.\\
&\quad^\forall i\in\{1,\ldots, N\},
}
}
 where $p_i\in [0, 1]$ are the parameters and $x_1,\ldots, x_N$ are independent. 

The feature means over the candidate distribution have an analytic expression such that the gradient can be obtained easily: 
\AL{
\SP{
\et_{ij}(q)&=\expe{q}{\phi_{ij}(\vx)}=\expe{q}{x_ix_j}\\
&=(2p_i-1)(2p_j-1).
}
}

We applied the variable transformation $s_i=\log\frac{p_i}{1-p_i}$ and optimized the parameter set $(s_1,\ldots,s_N)\in \Real^N$ in the algorithm. This variable transformation states the relationship between the gradients as $\pp{s_i}f=\frac{dp_i}{ds_i}\pp{p_i}f$. However, the coefficient $\frac{dp_i}{ds_i}=p_i(1-p_i)$ can be very small if $\lvert s_i \rvert$ is large. We dropped this coefficient, which did not affect the sign of the update, and used $\pp{p_i}f$ as the gradient. 

The candidate distribution for the random jump is obtained by setting the sign of each $s_i$ randomly, keeping the absolute value $\lvert s_i \rvert$.  

\subsection{  UCI wine quality data }

The input data were split into training and validation sets. Twenty percent (20\%) of the data for red and white wine each, were randomly selected as the validation set. The remaining training set was used to obtain the feature statistics needed by the algorithm.  After model validation, we reported the results for $\lambda=200$ in this study.

The parameterization of the candidate distributions $\qset$ is similar to the case of the bimodal distribution above. We used the candidate distribution set $\qset$, defined as consisting of the set of the following random variables: 
\AL{
&x_i \sim \mathcal{N}(\mu_i, \sigma_i^2)\quad^\forall i\in\{1,\ldots,11\},
}
 where $\mu_i\in \Real$ and $\sigma_i > 0.01$ for all $i\in\{1,\ldots,11\}$ are the parameters and $x_1,\ldots, x_{11}$ are independent. 

The means of the features $\phi_i^{(1)}$, $\phi_i^{(3)}$, and $\phi_i^{(4)}$ over the candidate distribution can be obtained in the same way as in the case of the bimodal distribution above. In addition, the mean of $\phi^{(2)}_{ij}$ is calculated as follows: 
\AL{
\SP{
\et^{(2)}_{ij}(q)&=\expe{q}{\phi_{ij}^{(2)}(\vx)}=\expe{q}{x_i^2}\\
&=\mu_i^2+\sigma_i^2
}& (i=j),\\
\SP{
\et^{(2)}_{ij}(q)&=\expe{q}{\phi_{ij}^{(2)}(\vx)}=\expe{q}{x_ix_j}\\
&=\mu_i\mu_j
}& (i\neq j).
}
 They also have analytic expressions, allowing the gradients to be obtained easily.

The variable transformation $l_i=\log\sigma_i$ was applied to the optimizaiton algorithm. The candidate distribution for the random jump is taken for each $(\mu_i,l_i)$ as in the case of the bimodal distribution above. 

The plots of the distributions in Figs.~\ref{wine_pp_part}, \ref{wine_circles} and \ref{wine_violin} are made from sufficient number of random samples drawn from the output distribution $\poutput$. 

\section{ Full plot of the UCI wine quality data  }
\label{sec-full-pairplot}
The pair plot for all variables of the UCI wine quality data is displayed in Figs. \ref{wine_pp_red} and \ref{wine_pp_white}, respectively. 

\begin{lfigure}[ htbp ]
\begin{center}
 
\begin{minipage}[b ]{0.6\textwidth} {\includegraphics[ width=\textwidth ]{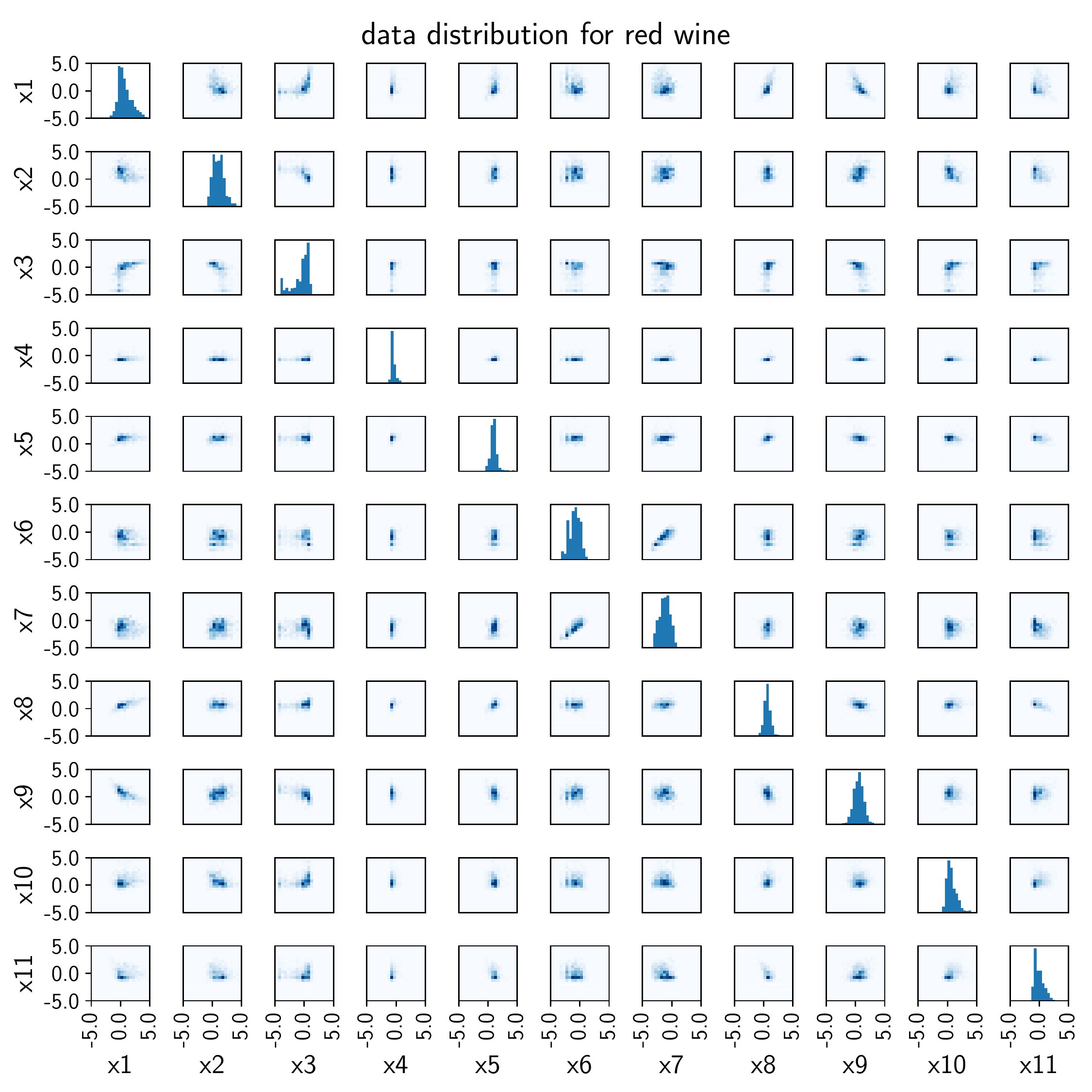} \subcaption{} \label{wine_pp_data_red} }\end{minipage} \\ 
\begin{minipage}[b ]{0.6\textwidth} {\includegraphics[ width=\textwidth ]{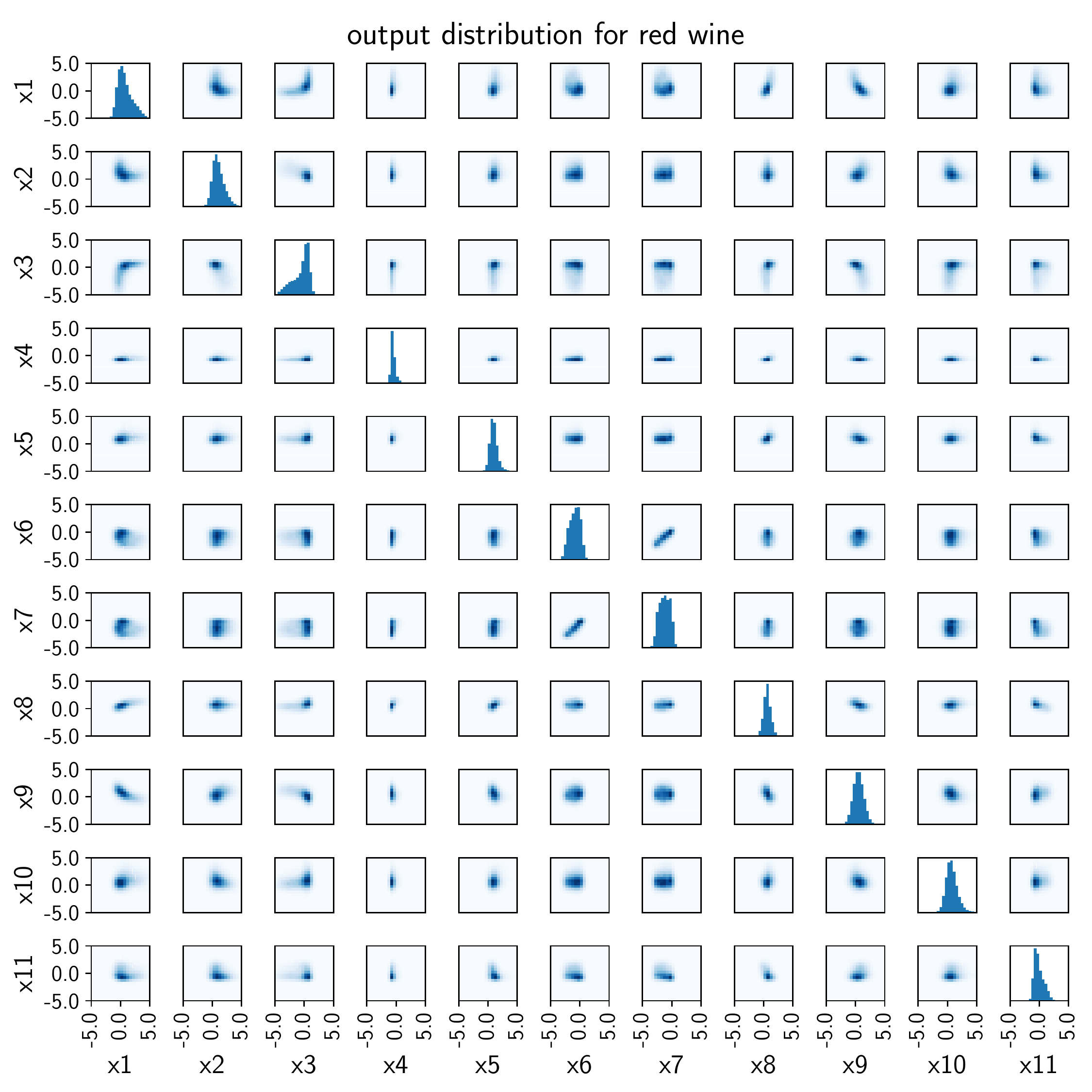} \subcaption{} \label{wine_pp_herding_red} }\end{minipage} 
\end{center}
\caption{For the red wine, the pair plot of the distributions of the dataset and distribution obtained from the entropic herding is shown. \subref{wine_pp_data_red} The dataset. \subref{wine_pp_herding_red} The distribution obtained by entropic herding. The meaning of the plot is the same as that of Fig.~\ref{wine_pp_part}, but all variables in the model are displayed}
\label{wine_pp_red}
\end{lfigure} 

\begin{lfigure}[ htbp ]
\begin{center}
 
\begin{minipage}[b ]{0.6\textwidth} {\includegraphics[ width=\textwidth ]{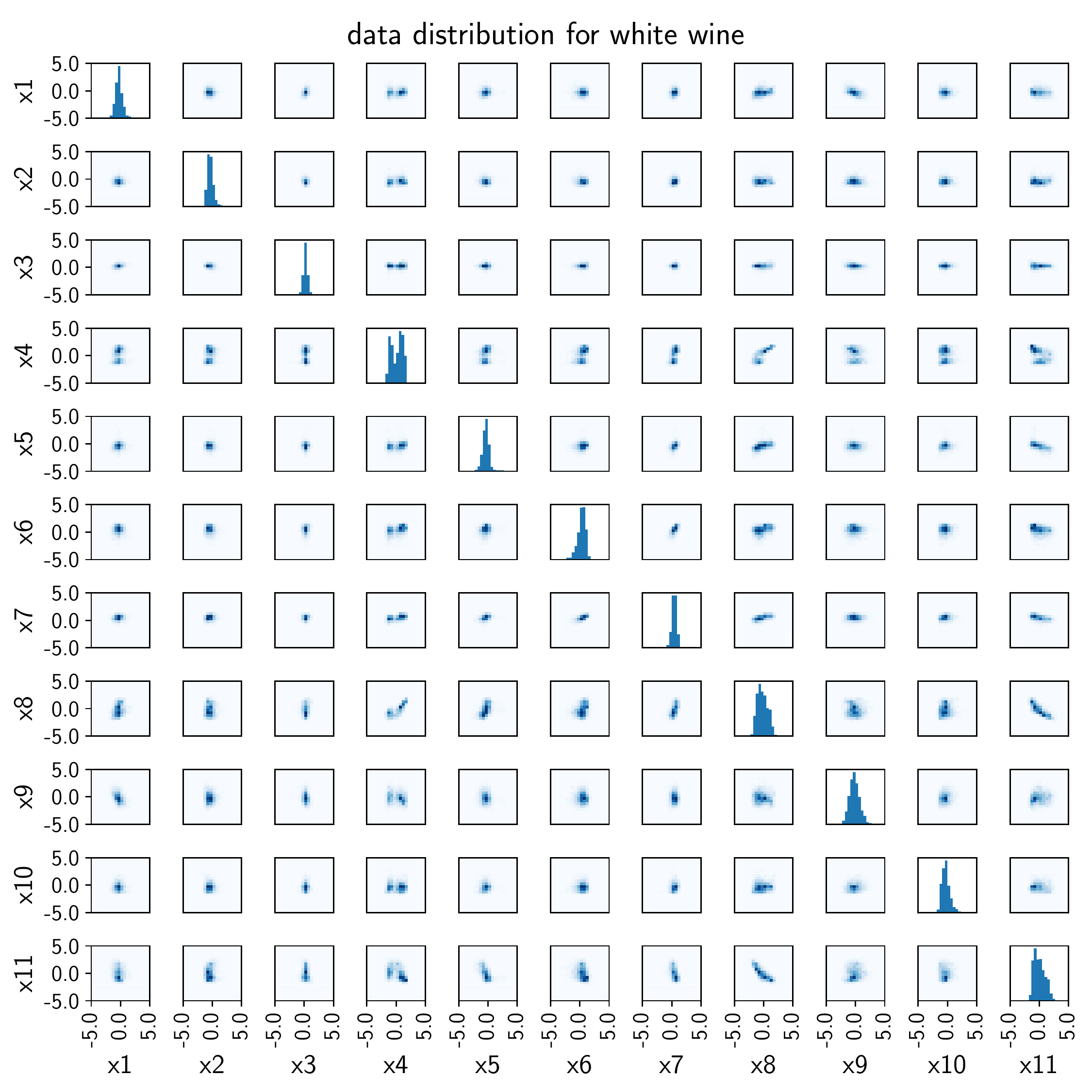} \subcaption{} \label{wine_pp_data_white} }\end{minipage} \\ 
\begin{minipage}[b ]{0.6\textwidth} {\includegraphics[ width=\textwidth ]{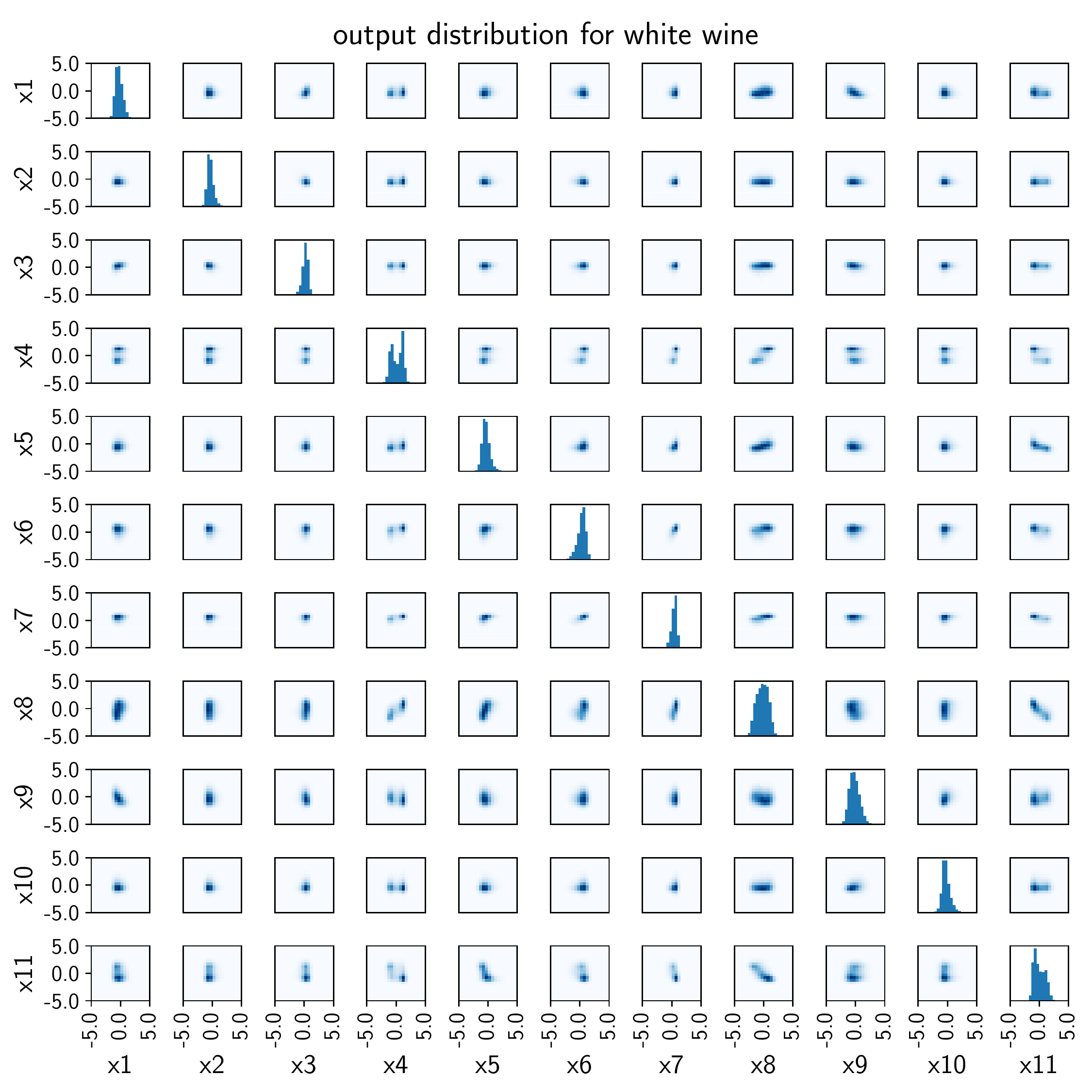} \subcaption{} \label{wine_pp_herding_white} }\end{minipage} 
\end{center}
\caption{For the white wine, the pair plot of the distributions of the dataset and the distribution obtained from entropic herding is shown. \subref{wine_pp_data_white} The dataset. \subref{wine_pp_herding_white} The distribution obtained by entropic herding. The meaning of the plot is the same as that of Fig.~\ref{wine_pp_part}, but all variables in the model are displayed}
\label{wine_pp_white}
\end{lfigure}

\section{ The optimum for the $\losa$ minimization  }
\label{sec-proof}
In this section, we present the conditions for the existence of a solution for \eq{eq-tho}. Let us assume that $\xset$ is discrete or $\xset\subset\Real^N$, and let $\pm:\xset\to\Real$ be a measurable function for each $m\in\mset$. Let $\eng_\vth, Z_\vth$, and $\pi_\vth$ be the energy function, partition function, and probability density function, respectively. In addition, let $\eta_m(\vth)\equiv\eta_m(\pi_\vth)$ denote the expectation of the feature on $\pi_\vth$ for simplicity, because we only consider the Gibbs distribution in this section. Namely, for $\xset\subset \Real^N$, we define  
\AL{
\eng_\vth(x)&=\summ \thm\pm(x),\LN{eq-energy-appx}
Z_\vth&=\int_{\xset} \exp\AS{-\eng_\vth(x)}dx,\LN{eq-z-appx}
\pi_\vth(x)&=\frac{1}{Z_\vth}\exp\AS{-\eng_\vth(x)},\LN{eq-gibbs-appx}
\fm(\vth)&=\frac{1}{Z_\vth}\int_{\xset}\ph_m(x)\exp\AS{-\eng_\vth(x)} dx\LB{eq-eta-appx}.
}
 Let $\domc\subset\mdim$ be the set of parameter vectors $\vth$ such that the integrals in \eq{eq-z-appx} and (\ref{eq-eta-appx}) are finite. For discrete $\xset$, we similarly define them by replacing the integral with the summation.  For an $\numm$-dimensional closed rectangle $\domb=[\lbb_1, \ubb_1]\times\cdots\times[\lbb_M, \ubb_M]$, let 
\AL{
\domblb_m&=\setpar{\vth\in\domb}{\th_m=\lbb_m},\NN 
\dombub_m&=\setpar{\vth\in\domb}{\th_m=\ubb_m}.
}
 Let us assume that $\lm>0$ for each $m\in\mset$. The main results of this section are as follows:

\begin{theorem}
\label{th:first} Let $\domb\subset \domc$ be an $\numm$-dimensional closed rectangle such that $\domb=[\lbb_1, \ubb_1]\times\cdots\times[\lbb_M, \ubb_M]$, and let $\fm$ be continuous on $\domb$ for each $m\in \mset$. Let $\etalba_m,\etauba_m$ be such that 
\AL{
\fm(\vth)&\ge\etalba_m \quad\text{for all}\quad \vth\in\domblb_m\text{ and}\NN 
\fm(\vth)&\le\etauba_m \quad\text{for all}\quad \vth\in\dombub_m
}
 for each $m\in \mset$. Then, for $\vec{\mu}\in\mdim$ such that  
\AL{
\etauba_m-\frac{\ubb_m}{\lm}\le \mm \le \etalba_m-\frac{\lbb_m}{\lm}\LB{eq-mu-condition}
}
 holds for all $m\in \mset$, there exists a solution of  
\AL{
\thm=\lm(\fm(\vth)-\mm)\quad{}^\forall m\in \mset.\LB{eq-tho-appx}
}
 
\end{theorem}

\begin{example}
[Normal distribution]\label{ex:normal} Let $\xset=\Real,\numm=2$ and $(\ph_1(x),\ph_2(x))=(x,x^2)$. Let $\domb=[-C, C]\times[1/(2s_{max}^2), 1 /(2s_{min}^2)]$, where $C>0$ and $0<s_{min}<s_{max}$. Let $(m,s^2)=(-\th_1/(2\th_2), 1/(2\th_2))$. Then, the mean and variance of $x$ for distribution $\pi_\vth$ are $m$ and $s^2$, respectively, because the energy function is  
\AN{
E_\vth(x)&=\th_1x+\th_2x^2\\
&=-\frac{m}{s^2}x+\frac{1}{2s^2}x^2\\
&=\frac{1}{2s^2}\AS{(x-m)^2-m^2}.
}
 Then, for $\vth\in\domb$, it holds 
\AN{
\et_1(\vth)&=\expe{\pi_\vth}{x}=m&\\
&=\begin{cases}
    \frac{C}{2\th_2}\ge 0& (\th_1=-C),\\
    -\frac{C}{2\th_2}\le 0& (\th_1=C),\NN
\end{cases}\\
\et_2(\vth)&=\expe{\pi_\vth}{x^2}=m^2+s^2&\\
&=\begin{cases}
    \th_1^2s_{max}^2+s_{max}^2\ge s_{max}^2& (\th_2=\frac{1}{2s_{max}^2}),\NN
    \th_1^2s_{min}^2+s_{min}^2\le (1+C^2)s_{min}^2&(\th_2=\frac{1}{2s_{min}^2}).
\end{cases}
}
 Then, we can use the Theorem \ref{th:first} with parameters $(A_1,B_1,A_2,B_2)=(0,0,s_{max}^2,(1+C^2)s_{min}^2)$; The condition \eq{eq-mu-condition} becomes 
\AN{
-\frac{C}{\Lambda_1}&\le\mu_1\le\frac{C}{\Lambda_1},\NN
(1+C^2)s_{min}^2-\frac{1}{2s_{min}^2\Lambda_2}&\le\mu_2\le s_{max}^2-\frac{1}{2s_{max}^2\Lambda_2}.
}
 For any $(\mu_1,\mu_2)\in\Real^2$, because we can select $\domb$ so as to satisfy these conditions, there exists a solution of \eq{eq-tho-appx}.  
\end{example}
To prove Theorem \ref{th:first}, we use the Poincare--Miranda theorem, which was first studied \citet{Poincare1883,Poincare1884} and proved by \citet{Miranda1940} (see also \mycite{kulpaPM,FPTbook}), as follows:

\begin{theorem}
[Poincare-Miranda ((C.3) in \mycite{FPTbook}, p.100)]\label{th:pm} Let $\bj^N$ be an $N$-dimensional cube $\setpar{\vec{x}=(x_1,\ldots,x_N)}{\lvert x_i\rvert\le 1 \text{ for } i=1,\ldots,N}$, the $i$th face $\setpar{\vec{x}\in\bj^N}{x_i=+1}$ is denoted by $\bj_i^+$, and the opposite face $\setpar{\vec{x}\in\bj^N}{x_i=-1}$ is denoted by $\bj_i^-$. Let $f_1,\ldots,f_N$ be continuous real-valued functions on $\bj^N$ such that for each $i=1,\ldots,N$, 
\AL{
\quad f_i(\vx)\ge 0 \quad \text{for } \vx\in\bj_i^+,\quad f_i(\vx)\le 0 \quad \text{for } \vx\in\bj_i^-,
}
 then, there exists $\hat{\vx}$ such that $f_i(\hat{\vx})=0$ for each $i=1,\ldots,N$. 
\end{theorem}

\begin{proof}
[Proof of Theorem \ref{th:first}] Let $\vth_\vx$ be a vector such that $(\vth_\vx)_m=\lbb_m+(\ubb_m-\lbb_m)(x_m+1)/2$. Then, it holds that $\vth_\vx\in\domblb_m$ for $\vx\in\bj^-_i$ and $\vth_\vx\in\dombub_m$ for $\vx\in\bj^+_i$. Let $f_m(\vx)=(\vth_\vx)_m-\lm(\fm(\vth_\vx)-\mm)$ for each $m\in \mset$. Then, for each $m\in \mset$ and for $\vx\in\bj^-_m$, it holds  
\AL{
f_m(\vx)&=\lbb_m-\lm(\fm(\vth_\vx)-\mm) \NN
&\le \lbb_m - \lm\AS{\etalba_m-\AS{\etalba_m-\frac{\lbb_m}{\lm}}}\NN
&=0,
}
 and for each $m\in \mset$ and for $\vx\in\bj^+_m$, it holds  
\AL{
f_m(\vx)&=\ubb_m-\lm(\fm(\vth_\vx)-\mm) \NN
&\ge\ubb_m-\lm\AS{\etauba_m-\AS{\etauba_m-\frac{\ubb_m}{\lm}}}\NN
&=0.
}
 Then, by Theorem \ref{th:pm}, there exists $\hat{\vx}$ such that $f_m(\hat{\vx})=0$ for each $m\in \mset$; therefore, $\vth_{\hat{\vx}}$ is the solution of \eq{eq-tho-appx}. 
\end{proof}
The conditions (\ref{eq-mu-condition}) are likely to be relaxed by decreasing $\lbb_m$ and increasing $\ubb_m$. However, it is not always the case because $\etalba_m$, which is the lower bound of the feature value on the face of $\domb=[\lbb_1, \ubb_1]\times\cdots\times[\lbb_M, \ubb_M]$, may decrease when $\domb$ is expanded. The conditions (\ref{eq-mu-condition}) will be satisfied, if we can sufficiently expand $\domb$ without changing the bounds $\etalba_m$ and $\etauba_m$ that appears in the condition. Based on this strategy, the following condition ensures the existence of a solution for any $\vec{\mu}\in\mdim$.

\begin{theorem}
\label{th:second} If $\domc=\mdim$ and, for each $m\in \mset$, it holds both \begin{itemize} \item there exists $\etalbb_m\in \Real$ such that $\fm(\vth)\ge \etalbb_m$ for all $\vth\in\mdim$, and \item there exists $\etaubb_m\in \Real$ such that $\fm(\vth)\le \etaubb_m$ for all $\vth\in\mdim$, \end{itemize} then, there exists a solution of \eq{eq-tho-appx} for any $\vec{\mu}\in\mdim$. 
\end{theorem}

\begin{proof}
 Let $\lbd_m, \ubd_m$ be such that $\lbd_m<\ubd_m$ for each $m\in\mset$. Let $\etalba_m=\etalbb_m$, $\lbb_m=\min\setp{\lm(\etalba_m-\mm), \lbd_m}$, $\etauba_m=\etaubb_m$, and $\ubb_m=\max\setp{\lm(\etauba_m-\mm), \ubd_m}$ for each $m\in\mset$. Let $\domb=[\lbb_1, \ubb_1]\times\cdots\times[\lbb_M, \ubb_M]\subset\domc$ be an $\numm$-dimensional closed rectangle. Then, the existence of a solution of \eq{eq-tho-appx} for any $\vec{\mu}\in\mdim$ is assured by Theorem \ref{th:first} because the conditions are satisfied as follows:     Note that $\fm(\vth)\ge\etalba_m$ for $\vth\in\domblb_m$ under the assumption. Because $\lbb_m\le\lm(\etalba_m-\mm)$, it also holds that 
\AL{
\mm \le \etalba_m-\frac{\lbb_m}{\lm},\LB{eq-mu-condition2}
}
 which is the second inequality of \eq{eq-mu-condition}.     Similarly, note that $\fm(\vth)\le\etauba_m$ for $\vth\in\dombub_m$ under the assumption. Because $\ubb_m\ge\lm(\etauba_m-\mm)$, it also holds that 
\AL{
\etauba_m-\frac{\ubb_m}{\lm}\le \mm,\LB{eq-mu-condition1}
}
 which is the first inequality of \eq{eq-mu-condition}. 
\end{proof}
 The following are simple cases for Theorem \ref{th:second}, where both $\xset$ and $\pm$ are bounded.

\begin{corollary}
\label{th:coro} If either 
\begin{itemize}
\item $\xset$ is discrete and finite, or
\item $\xset$ is a bounded closed subset of $\mdim$ and, for all $m\in \mset$, $\ph_m(x)$ is a bounded function on $\xset$,
\end{itemize}
 then, there exists a solution of \eq{eq-tho-appx} for any $\vec{\mu}\in\mdim$. 
\end{corollary}
 Note that if $\xset$ is unbounded, the distribution \eq{eq-eta-appx} is not well-defined for all $\vth\in\mdim$, that is, $\domc\neq \mdim$. Therefore, the existence of the solution of \eq{eq-tho} is not guaranteed for all $\vec{\mu}\in\mdim$, because $\domb\subset\domc$ that satisfies the conditions of Theorem \ref{th:first} for all $\vec{\mu}\in\mdim$ may not exist. However, we still have a chance to assure the existence of a solution by explicitly obtaining $\domb$, which may depends on $\vec{\mu}$, as illustrated in Example \ref{ex:normal}.

\end{appendices}


\bibliography{mybib}


\end{document}